\documentclass[]{fairmeta}

\usepackage[utf8]{inputenc} 
\usepackage{lmodern}       
\usepackage[T1]{fontenc}    
\usepackage{hyperref}       
\usepackage{url}            
\usepackage{booktabs}       
\usepackage{nicefrac}       
\usepackage{microtype}      
\usepackage{xcolor}         
\usepackage{graphicx}
\usepackage{subcaption}  
\usepackage{caption}

\title{Convex Dominance in Deep Learning I: A Scaling Law of Loss and Learning Rate}

\author[1]{Zhiqi Bu}
\author[2,\dagger]{Shiyun Xu}
\author[2,\dagger]{Jialin Mao}
\contribution[\dagger]{Joint second authors.}

\affiliation[1]{FAIR, Meta Superintelligence Labs}
\affiliation[2]{Independent Researcher}
\correspondence{\email{zhiqibu@meta.com}}

\usepackage{verbatim}
\usepackage{url}
\usepackage{algorithm}
\usepackage{algorithmicx}
\usepackage[noend]{algpseudocode}
\usepackage{amsmath, amsfonts, amssymb, amsthm, amsbsy, amscd, bm, bbm,mathrsfs} 
\usepackage{amssymb}
\usepackage{graphicx}
\usepackage{setspace}
\usepackage{hyperref}
\usepackage{appendix}
\usepackage{xspace}
\usepackage{epigraph}
\usepackage[thinc]{esdiff}
\graphicspath{{figs/}}

\usepackage[labelfont=bf]{caption}
\usepackage{cleveref}

\renewcommand{\arraystretch}{1.5}

\makeatletter
\newcommand*{\rom}[1]{\expandafter\@slowromancap\romannumeral #1@}
\makeatother

\makeatletter
\newcommand{\be}{\begin{equation}}
\newcommand{\ee}{\end{equation}}

\newcommand{\E}{\mathbb{E}}

\allowdisplaybreaks \numberwithin{equation}{section}
\makeatother

\newcommand{\w}{\mathbf{w}}
\newcommand{\g}{\mathbf{g}}
\newcommand{\R}{\mathbb{R}}
\newcommand{\etap}{\eta_\textup{peak}}
\newcommand{\etar}{\eta_\textup{ref}}

\newtheorem{theorem}{Theorem}
\newtheorem{othertheorem}{othertheorem}[section]

\newtheorem{corollary}[othertheorem]{Corollary}

\newtheorem{remark}[othertheorem]{Remark}
\newtheorem{example}[othertheorem]{Example}
\newtheorem{abstraction}{Generalization}

\newtheorem{condition}[othertheorem]{Condition}

\crefname{othertheorem}{condition}{conditions}
\Crefname{othertheorem}{Condition}{Conditions}

\crefname{abstraction}{generalization}{generalizations}
\Crefname{abstraction}{Generalization}{Generalizations}

\abstract{
Deep learning has non-convex loss landscape and its optimization dynamics is hard to analyze or control. Nevertheless, the dynamics can be empirically convex-like across various tasks, models, optimizers, hyperparameters, etc. In this work, we examine the applicability of convexity and Lipschitz continuity in deep learning, in order to precisely control the loss dynamics via the learning rate schedules. We illustrate that deep learning quickly becomes weakly convex after a short period of training, and the loss is predicable by an upper bound on the last iterate, which further informs the scaling of optimal learning rate. 
Through the lens of convexity, we build scaling laws of learning rates and losses that extrapolate as much as $80\times$ across training horizons and $70\times$ across model sizes.
}

\begin{document}

\maketitle

\section{Introduction}
Deep learning has highly non-convex and complex loss landscape, e.g. the global minimum may not be unique, and there may be many local minima and saddle points where the optimization can be trapped \citep{garipov2018loss,choromanska2015loss,dauphin2014identifying,jin2017escape,sagun2016eigenvalues}. Nevertheless, the empirical success of optimization in deep learning has implied that some benign properties may hold and be leveraged.

In fact, there has been a long history of empirically exploring convexity in deep learning. For example, Llama training (non-convex) with AdamW \citep{loshchilovdecoupled} is closely similar to convex optimization with SGD, in terms of the shapes of loss curves in Figure 1 of \citep{schaippsurprising}. Empirical evidence by \citep{zhousgd} has shown that SGD follows a star-convex path during the optimization of neural networks. In addition, convexity is universally observed along the direction of gradients in vision and language models by \citep{bu2025gradient}.

Another line of research has focused on
the two-dimensional loss landscape, by \citep{li2018visualizing,im2016empirical} (along two random and normalized directions) and \citep{allen2019convergence} (along the gradient and  the negatively curved direction of the Hessian), which is approximately locally convex for residual neural networks with various depth and width. Furthermore, these convex-like loss landscapes are also observed on large language models such as RoBERTa, LLaMA, Qwen, and Mistral in \citep{zhong2022improving,chen2025understanding,lee2024fp8}.

Besides these low-dimensional loss landscapes,
the Hessian spectrum provides a rigorous local notion of curvature: a positive-definite Hessian (all eigenvalues positive) indicates that the loss landscape is locally convex around that point. Empirical observations show that at initialization, the Hessian often contains many large negative eigenvalues, which quickly shift toward zero and become much smaller in magnitude than the positive eigenvalues. Consequently, the spectrum becomes dominated by positive eigenvalues and the loss landscape becomes approximately convex \citep{yao2020pyhessian,papyan2018full,sankar2021deeper,zhang2024transformers}.

In addition, theoretical analysis shows that two-layer and deeper neural networks can be regarded as convex in the wide-network limit, under neural tangent kernel \citep{jacot2018neural,lee2019wide,du2018gradient,du2019gradient,li2018learning,allen2019learning}, neural repopulation \citep{fang2019over,fang2022convex}, and mean field \citep{mei2018mean,chizat2018global}.

As a matter of fact, many works in deep learning have been drawing useful understanding from convex analysis to non-convex deep learning: \cite{sutskever2013importance} shows that momentum significantly accelerates the convergence; \cite{defazio2023optimal} explains the effectiveness of learning rate warmup and decay from a convex viewpoint; the effect of weight decay \citep{krogh1991simple} is first understood in ridge regression \citep{hoerl1970ridge}.

Specifically,the relationship between the learning rate sequence $\{\eta_t\}$ and the upper bounds of loss sequence $\{L_t\}$ has been heavily studied. These bounds come from different settings, studying convex or strongly convex loss, Lipschitz continuous or smooth loss, finite-iteration or asymptotic bound, averaged or last iterate, etc. Our work is directly based on Corollary 12 of \citep{defazio2023optimal} (re-stated in \eqref{eq:last lr array}), and a simplified version can be found in \eqref{eq:lr series}, which already provides some insights as we will summarize in \Cref{rem:d-adapt}.

\subsection{Related work}
\paragraph{Convex to deep learning.}
It is well-known in convex regime that loss can converge at $O(1/\sqrt{T})$ and optimal learning rate is $O(1/\sqrt{T})$. However, whether or when these conclusions hold in deep learning is largely unclear. From this perspective, our work is most closely related to \citep{schaippsurprising}, both following from \citep{defazio2023optimal}. In contrast, we not only extend major findings of \citep{schaippsurprising} theoretically (e.g. our qualifying exam in \Cref{def:qualifying} covers any schedule), but also focus more on empirical validation. For example, $O(1/\sqrt{T})$ convergence of loss and consistent patterns across models and optimizers in our Figures 6-12 are not presented in \citep{schaippsurprising}, which focuses on one model and single training horizon like our Figures 2-5. In particular, our empirical approach is data-driven and practically applicable in large-scale deep learning. 

\paragraph{Scaling laws.}
Current popular scaling laws are primarily about loss \citep{kaplan2020scaling,hoffmann2022training} (i.e. scaling laws of loss), predicting how loss changes as model sizes and training horizons change, assuming optimal learning rate. As a result, learning rate is not explicitly presented in these laws. Other laws (i.e. scaling laws of learning rate) can scale learning rate across training horizons $\etap^*=\lambda T^{-\alpha}$, where $\alpha=0.125$ in \citep{bi2024deepseek}, $\alpha\in\{0.32, 0.38, 0.42,0.65,0.70\}$ by Table 5 in \citep{bjorck2024scaling}, and some are horizon-unaware in \citep{porian2024resolving,wang2024scaling} (i.e. $\alpha=0$). Nevertheless, these laws may not predict loss. In this work, we propose a law that simultaneously predict loss and optimal learning rate for the fixed value $\alpha=0.5$.

\paragraph{Learning rate transfer.} Maximal update parameterization (muP) \citep{yang2022tensor} is a technique to transfer optimal learning rate across model size. While Table 1 of \citep{yang2022tensor} claimed it also transfers across training horizons, this contradicts with empirical evidence in \citep{bjorck2024scaling} (see Section 3.3) and our analysis.


\subsection{Contributions}
In this work, we study the scaling law of deep learning loss and learning rate, through the lens of convex loss (non-smooth) and bounded gradient. We will establish a series of generalizations from convex theory to deep learning, which is presented in \Cref{tab:generalization} and summarized as follows.
\begin{enumerate}
    \item We study the convex-like behaviors in deep learning for general model architectures, optimizers and learning rate schedules, hence establishing a non-asymptotic mapping from learning rate sequence to loss sequence.
    \item We generalize to an asymptotic upper bound of loss, achieving $O(1/\sqrt{T})$ convergence when (I) the peak learning rate is scaled by $1/\sqrt{T}$ and (II) the learning rate schedule is qualified.
    \item We propose a data-driven method to fit the asymptotic bound, establishing a scaling law across training horizons and model sizes.
\end{enumerate}

\section{Convergence of SGD under convex loss}
\label{sec:convex sgd}
\subsection{Revisiting non-asymptotic bound of SGD}
We consider the stochastic gradient descent (SGD) as
$\w_{t+1}=\w_t-\eta_{t+1} \g(\w_t)$,
where $\w$ is the parameters, $\eta$ is the learning rate, $\g$ is the mini-batch gradient with $\E\g(\w)=\nabla  L$, $0\leq t< T$ is the iteration and $T$ is the training horizon (i.e. number of iterations). The learning rate is defined by two factors via $\etap\cdot s_t(T)$: (I) the learning rate schedule, which is a function $s_t\in [0,1]$ (e.g. linear decay is $s_t(T)=1-t/T$), and (II) the peak learning rate $\etap\in\R^+$ which is a positive scalar.

We briefly review the convergence analysis under the convex and bounded gradient conditions.
\begin{condition}
\label{def:convex and Lipschitz}
Denoting a differentiable function as $L$ and its gradient as $\nabla L$, then $L$ is
\begin{align}
\textit{convex if }\forall (\w,\bm{x}), L(\w)-L(\bm{x})\leq (\w-\bm{x})^\top \nabla L(\w)
\\
\textit{bounded in gradient if }\exists G \text{ s.t. }\forall \w, \E\|\g(\w)\|^2\leq G^2
\end{align}
\end{condition}

\begin{remark}
In \Cref{def:convex and Lipschitz}, the convexity is not necessary for our analysis as it can be replaced by star-convexity and iterate-wise convexity along the optimization path, i.e. 
$$L(\w_t)-L(\w_*)\leq (\w_t-\w_*)^\top \nabla L(\w_t)$$ where $\w_*\in\text{argmin}_\w L(\w)$ is the minimizer, and 
$$L(\w_t)-L(\w_s)\leq (\w_t-\w_s)^\top \nabla L(\w_t).$$
Additionally, the  bounded gradient condition reduces to Lipschitz continuity when SGD is full-batch.
\end{remark}
In the parameter space, with $\g_t:=\g(\w_t)$,
\begin{align*}
\|\w_{t+1}-\w_*\|^2
&=\|\w_{t}-\w_*\|^2-2\eta_{t+1} (\w_{t}-\w_*)^\top \g_t +\eta_{t+1}^2 \|\g_t\|^2
\end{align*}
For bounded gradient and convex loss, denoting $L_*=\min_\w L(\w)$, we have in expectation
\begin{align*}
\E\|\w_{t+1}-\w_*\|^2&\leq \E\|\w_{t}-\w_*\|^2-2\eta_{t+1} (L_{t}-L_*)+\eta_{t+1}^2 G^2
\end{align*}
Telescoping sum gives
\begin{align*}
0&\leq\E\|\w_{\tau}-\w_*\|^2\leq \|\w_{0}-\w_*\|^2-2\sum_{t=0}^{\tau-1}\eta_{t+1} (L_{t}-L_*)+\sum_{t=0}^{\tau-1}\eta_{t+1}^2 G^2
\end{align*}
Dividing by $2\sum_{t}\eta_{t+1}$ and applying Jensen's inequality, we obtain an upper bound of $L(\bar \w_\tau)$, where $\bar\w_\tau=\frac{\sum_{t=0}^{\tau-1}\eta_{t+1} \w_t}{\sum_{t=0}^{\tau-1} \eta_{t+1}}$ is the averaged iterate and $D:=\|\w_{0}-\w_*\|$,
\begin{align}
\E L(\bar\w_\tau)\leq \frac{\sum_{t=0}^{\tau-1}\eta_{t+1} L_{t}}{\sum_{t=0}^{\tau-1} \eta_{t+1}}&\leq L_*+\frac{D^2}{2\sum_{t=1}^{\tau} \eta_{t}}+ \frac{G^2\sum_{t=1}^{\tau}\eta_{t}^2}{2\sum_{t=1}^{\tau} \eta_{t}}
:=L_\text{SGD-ave}(\{\eta_t\}).
\label{eq:lr series}
\end{align}

\begin{example}
\label{rem:d-adapt}
For constant learning rate $\eta$, the loss upper bound derived from convex analysis in \eqref{eq:lr series} simplifies to $L_*+\frac{D^2}{2T \eta}+ \frac{\eta G^2}{2}$.
This aligns with the empirical trade-off in deep learning that larger $\eta$ converges faster but to a higher loss, and vice versa [See Fig 10.14 \citep{deep_learning_bible_nnfs}]. Furthermore, this loss bound is minimized by $\eta_*=\frac{D}{\sqrt{T}G}$ in convex analysis, which underlies the fast convergence observed in deep learning such as D-adaptation \citep{defazio2023learning}, Prodigy \citep{mishchenko2023prodigy}, DoG \citep{ivgi2023dog}, and DoWG \citep{khaled2023dowg}. 
\end{example}

With one extra term in \cite[Corollary 12]{defazio2023optimal}, we have a bound of any single iterate:
\begin{align}
\E L(\w_\tau)\leq L_*+\frac{D^2}{2\sum_{t=1}^\tau \eta_t}+ \frac{G^2\sum_{t=1}^\tau\eta_t^2}{2\sum_{t=1}^\tau \eta_t}+\frac{ G^2}{2}\sum_{k=1}^{\tau-1}\frac{\eta_k}{\sum_{t=k+1}^\tau\eta_t}\frac{\sum_{t=k}^\tau \eta_t^2}{\sum_{t=k}^\tau \eta_t}
\label{eq:last lr array}
\end{align}

Note that \eqref{eq:lr series} and \eqref{eq:last lr array} translate an arbitrary learning rate sequence $\{\eta_t\}$ to an upper bound of the loss value. While both bounds shed some insights on the loss dynamics, we focus on the bound from \eqref{eq:last lr array}, since \eqref{eq:lr series} can be less precise in characterizing the loss curves (c.f. Figure 2 and Figure 13 in \citep{schaippsurprising}).



\subsection{Asymptotic loss bound at last iteration}
For any training horizon $T$, we can characterize the loss at the last iteration $\w_T$ via \eqref{eq:last lr array}. We show the upper bounds on the loss in terms of different learning rate schedules, the optimal peak learning rate $\etap^*$, and the optimal loss bound in \Cref{tab:theorem1}, where the results are derived in \Cref{col:last-iter} in \Cref{app:proofs}.

\begin{table}[!htb]
\centering
\caption{Convergence of optimal loss and optimal learning rate by \Cref{col:last-iter} under different schedules.}
\vspace{-0.3cm}
\resizebox{\linewidth}{!}{
    \begin{tabular}{c|c|c|c|c}
        learning rate schedule &$\eta_t$ formula& upper bound of $\E L(\w_T)$ &optimal loss bound& optimal $\etap^*(T)$  \\\hline
         constant & $\etap$ & $L_* + \frac{D^2}{2T \etap} + \frac{\etap G^2}{2} \ln T$& $L_* + D G \sqrt{\frac{\ln T}{T}}$ & $\frac{D}{G \sqrt{\ln T\cdot T}}$\\
         square-root inverse& $\etap/\sqrt{t}$ & $L_* + \frac{D^2}{4\sqrt{T}\etap} + \frac{\etap G^2\ln T}{4\sqrt{T}}$ & $L_*+D G \sqrt{\frac{\ln T}{4 T}}$ & $\frac{D}{G \sqrt{\ln T}}$\\
         linear decaying& $\eta_{\text {peak }}(1-t / T)$ & $L_*+\frac{D^2}{T \eta_{\text {peak }}}+\eta_{\text {peak }} G^2$ & $L_*+2 D G \sqrt{\frac{1}{T}}$ & $\frac{D}{G \sqrt{T}}$ \\
         cosine decaying& $\eta_{\text {peak }} \frac{1+\cos (\pi t / T)}{2}$ & $L_*+\frac{D^2}{T \eta_{\text {peak }}}+\eta_{\text {peak }} G^2 \cdot 1.061$ & $L_*+2 D G \sqrt{\frac{1.061}{T}}$ & $\frac{D}{G \sqrt{1.061 T}}$\\
          warmup-stable-decay& $\begin{cases}
             \etap &\text{if }t<cT\\
              \etap\frac{T-t}{T-cT}&\text{if }t\geq cT
          \end{cases}$
          &$L_*+\frac{D^2}{(1+c)T\etap}+ \etap G^2
\left[1+\frac{1}{2}\ln\left(\frac{1+c}{1-c}\right)\right]$& $L_* +  2DG \sqrt{\frac{1+\frac{1}{2}\ln\left(\frac{1+c}{1-c}\right)}{(1+c)T}}$ & $\frac{D}{G \sqrt{(1+c)\left(1+\frac{1}{2}\ln\left(\frac{1+c}{1-c}\right)\right) T}}$ 
          \end{tabular}
}
\label{tab:theorem1}
\end{table}

We summarize some insightful observations in \Cref{tab:theorem1}. 
\begin{itemize}
\item The optimal loss convergence rate is $O(1/\sqrt{T})$ and it can be achieved by some schedules like linear, cosine decay and warmup-stable-decay (WSD or trapezoid; \citep{xing2018walk,hagele2024scaling}), which we term as the \textit{qualified schedules}. In contrast, some other schedules like constant learning rate only achieve a suboptimal rate $O(\sqrt{\ln T/T})$.
\item The qualified schedules have common patterns: (1) the schedules are horizon-aware, i.e. $\eta_t$ is dependent on $T$, in constrast to the constant and square-root inverse schedules; (2) the loss bounds take the following form for some constants $q_1, q_2$:
\begin{align}
\E L(\w_T)\lesssim L_*+\frac{q_1^2}{T\etap}+\etap q_2^2 := L_\text{SGD-last}(\etap,T)
\label{eq:peak and loss}
\end{align}
\item The optimal learning rate of a qualified schedule is $O(1/\sqrt{T})$.
\end{itemize}
While \Cref{col:last-iter} have considered specific schedules and optimal peak learning rate, in what follows, we will generalize to non-optimal peak learning rates in \Cref{sec:non-optimal peak lr} and introduce an exam to select qualified schedules in \Cref{sec:qualify}.

\subsection{Optimal convergence under scaled learning rate}
\label{sec:non-optimal peak lr}
We show in \Cref{col:sgd ref} that any scaled $\etap$ can achieve the optimal $O(1/\sqrt{T})$ loss convergence, regardless of whether $\etap$ is optimal. Here, we scale $\etap$ by dividing any reference learning rate $\etar$ to $\sqrt{T}$.
\begin{corollary}
\label{col:sgd ref}
Consider a learning rate schedule and $\etap$ that satisfy \eqref{eq:peak and loss} for SGD under \Cref{def:convex and Lipschitz}.
\begin{enumerate}
    \item Any $\etar\in\R_+$ achieves the asymptotically optimal convergence rate:
\[
L_\text{SGD-last}(\etap=\etar/\sqrt{T},T) - L_* \sim Q(\etar)/\sqrt{T}=O\left(1/\sqrt{T}\right),
\]
in which we define $Q(\etar):={q_1^2}/\etar+\etar q_2^2$. 
\item Additionally, the optimal $\etar^*=\textup{argmin}_{\etar} Q=q_1/q_2$, and $L_\text{SGD-last}(T) - L_* \sim\frac{2q_1 q_2}{\sqrt{T}}$.
\end{enumerate}
\end{corollary}

\subsection{Qualifying exam for learning rate schedules}
\label{sec:qualify}

We \textit{qualify} a learning rate schedule function $s_t(T)$ if \eqref{eq:peak and loss} satisfies the following
\begin{align*}
L_\text{SGD-last}(\etap={1}/{\sqrt{T}},T) - L_*= O\left(1/{\sqrt{T}}\right),
\end{align*}
where we simply choose $\etar=1$ as supported by \Cref{col:sgd ref} part 1.

We highlight that the qualifying exam can be conducted in a training-free way without running any model optimization. We propose to employ symbolic analysis in \Cref{def:qualifying}, and the full derivation from discrete summation in \eqref{eq:last lr array} to continuous definite integrals can be found in \Cref{app:condition24}.
\begin{condition}
\label{def:qualifying}
We claim a learning rate schedule function $s_t(T)\in [0,1]$ is qualified, if the following holds with $\eta_t(T):=s_t(T)/\sqrt{T}$:
\begin{align*}
\frac{D^2}{2\int_0^T\eta_t dt}+\frac{ G^2}{2}\int_0^T \left(\frac{\eta_t^2}{ \int_{t}^T\eta_k dk}\right)dt=O(1/\sqrt{T}).
\end{align*}
\end{condition}



Theoretically, we prove in \Cref{col:last-iter} that 
\begin{itemize}
    \item linear decaying, cosine decaying, and WSD schedules pass the qualifying exam;
    \item constant and square-root inverse schedules fail.
\end{itemize}

Empirically, we validate the effectiveness of our qualifying exam in \Cref{fig:sgd multi}.

\begin{figure}[!htb]
\centering
\includegraphics[width=0.195\linewidth]{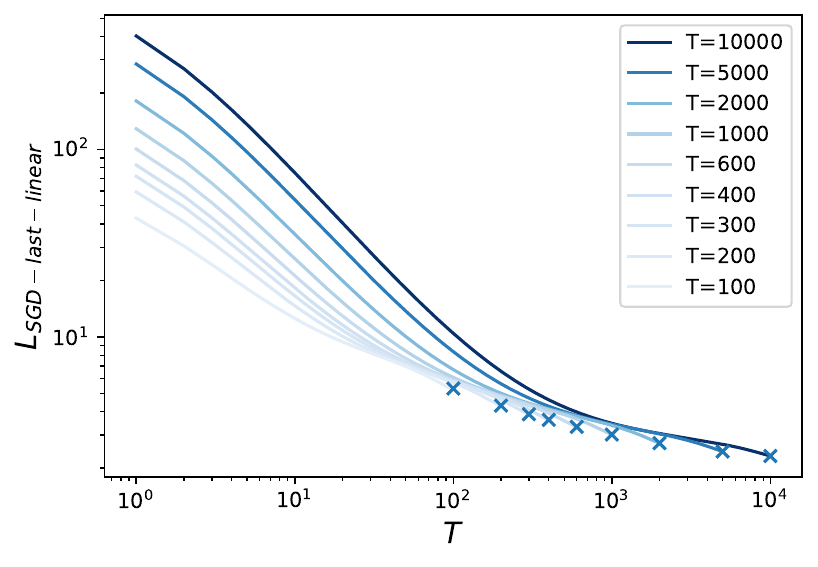}
\includegraphics[width=0.19\linewidth]{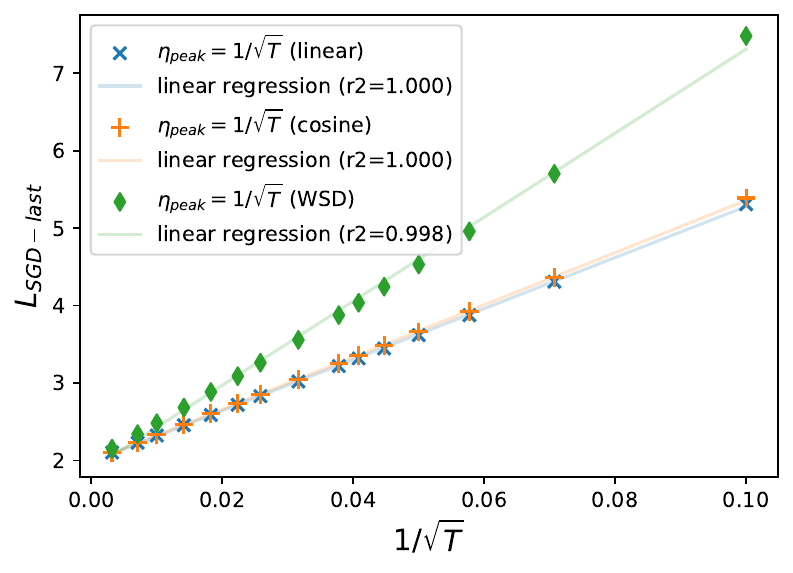}
\includegraphics[width=0.195\linewidth]{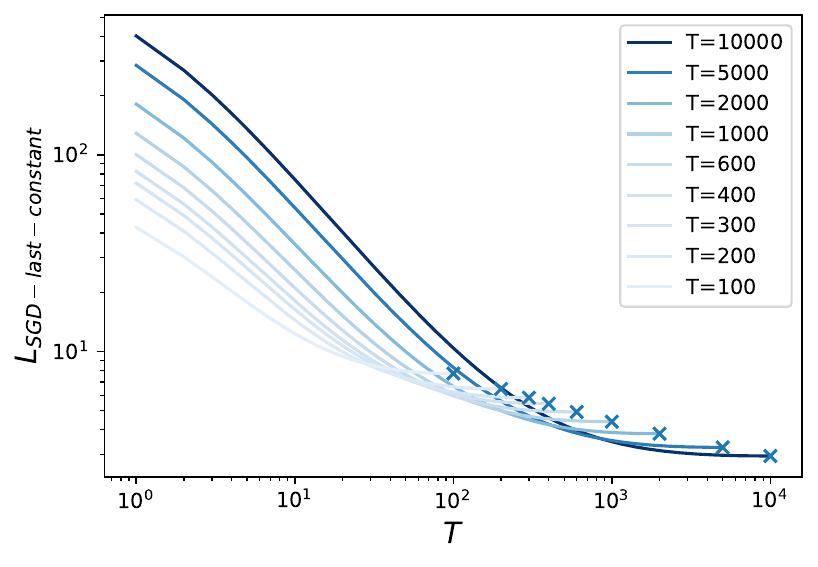}
\includegraphics[width=0.19\linewidth]{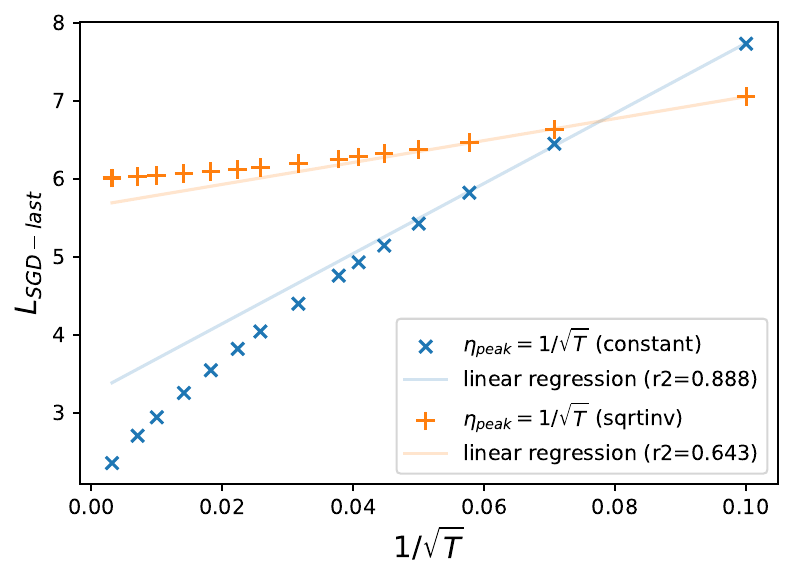}
\includegraphics[width=0.19\linewidth]{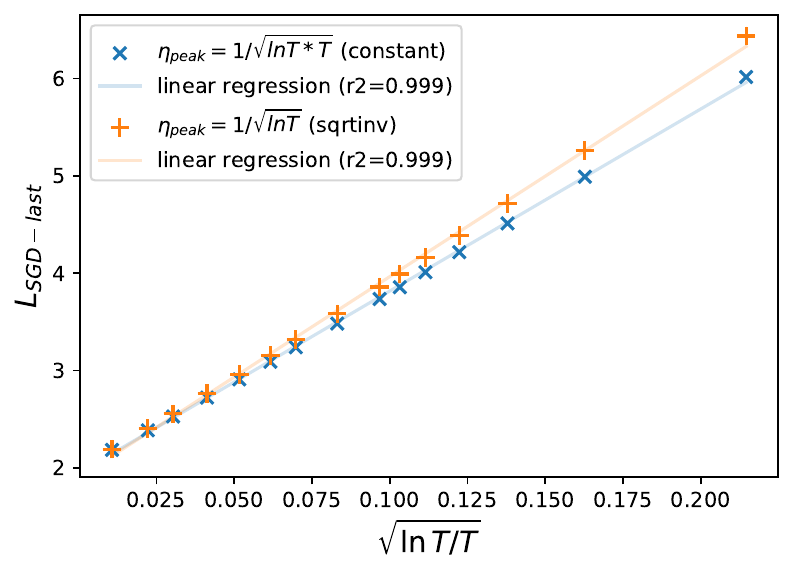}
\vspace{-0.3cm}
\caption{Upper bound of SGD loss in \Cref{eq:last lr array} with peak learning rate $\etap=1/\sqrt{T}$. Left-most is linear decaying schedule. Center is constant schedule.
}
\label{fig:sgd multi}
\end{figure}

\vspace{-0.3cm}
\section{Generalizations to deep learning and adaptive optimizers}
\label{sec:within run}

\subsection{Abstract form of any-iteration loss}

Our goal is to generalize the understanding \textbf{from convex analysis with SGD} beyond the scope of \eqref{eq:last lr array}, \textbf{to non-convex deep learning and to general optimizers}. However, we cannot use \eqref{eq:last lr array} directly as we do not know $D,G,L_*$. Alternatively, we will estimate these quantities in a data-driven manner and work with an abstract form of the loss in \eqref{eq:last lr array DL}.

\begin{abstraction}[generalized from \eqref{eq:last lr array}]
\label{abs:first}
For general optimizers and for deep learning, we have
\begin{align}
\E L(\w_\tau)\leq \tilde L_\infty+\frac{\tilde D^2}{2\sum_{t=1}^\tau \eta_t}+ \frac{\tilde  G^2}{2}\left(\frac{\sum_{t=1}^\tau\eta_t^2}{\sum_{t=1}^\tau \eta_t}+\sum_{k=1}^{\tau-1}\frac{\eta_k}{\sum_{t=k+1}^\tau\eta_t}\frac{\sum_{t=k}^\tau \eta_t^2}{\sum_{t=k}^\tau \eta_t}\right)
\label{eq:last lr array DL}
\end{align}
\end{abstraction}
In contrast to \eqref{eq:last lr array}, we require the following 
modifications:
\begin{itemize}
    \item The quantities $\tilde D$ and $\tilde G$ are no longer tied to $D=\|\w_0-\w_*\|$ and $G^2=\max_\w \E\|\g(\w)|\|^2$, i.e. we renounce the physical meaning of these quantities, because
    \begin{itemize}
        \item general optimizers (e.g. momentum SGD, AdamW, Muon, parameter-efficient training such as LoRA) are not covered by \eqref{eq:last lr array} ever under the convex setting in \Cref{def:convex and Lipschitz}.
        \item $D,G$ depend on the model architectures and datasets, which are hard to derive in practice. 
    \end{itemize}
    \item The irreducible loss is now $\tilde L_\infty=L(\lim_{\tau\to\infty}\w_\tau)$ instead of $L_*=L(\w_*)$. This accommodates the fact that there are infinite minima in deep learning, hence $\w_*$ is not unique and $L_*$ is not well-defined.
\end{itemize}

As we will observe in the following sections, \eqref{eq:last lr array DL} holds well at all iterations and becomes tight after some initial training. To be clear, we examine \eqref{eq:last lr array DL} on various
\begin{itemize}
    \item models: ResNet \citep{he2016deep}, ViT \citep{dosovitskiy2020image}, GPT2 \citep{radford2019language}, vision-language models (which use LLAMA3 architecture \citep{dubey2024llama} as the language backbone).
    \item tasks: ImageNet \citep{deng2009imagenet}, OpenWebText \citep{Gokaslan2019OpenWeb}, and Cauldron \citep{laurençon2024matters}.
    \item optimizers: SGD, AdamW \cite{loshchilovdecoupled}, Muon \citep{jordan2024muon}, parameter-efficient training (LoRA \citep{hu2022lora}), etc.
    \item learning rate schedules: linear decay, cosine decay, WSD, constant, and square-root inverse.
\end{itemize}

\subsection{Generalizing to deep learning}
We start with SGD (same as \eqref{eq:last lr array}) but in the deep learning setting. In \Cref{fig:resnet-imagenet}, we train ResNet18 on ImageNet dataset with SGD under 4 learning rate schedules. The goodness of fit\footnote{We fit a non-negative linear regression $\bm{y}\sim\bm{X}\bm{\beta}+\tilde L_\infty$ where $\bm{y}\in\R^T, \bm{\beta}=[\tilde D, \tilde G]^\top, \bm{X}\in\R^{T\times 2}$, and specifically $y_\tau= L(\w_\tau), X_{\tau,1}=\frac{1}{2\sum_{t=1}^\tau \eta_t}, X_{\tau,2}=\frac{1}{2}\left(\frac{\sum_{t=1}^\tau\eta_t^2}{\sum_{t=1}^\tau \eta_t}+\sum_{k=1}^{\tau-1}\frac{\eta_k}{\sum_{t=k+1}^\tau\eta_t}\frac{\sum_{t=k}^\tau \eta_t^2}{\sum_{t=k}^\tau \eta_t}\right)$.} for \eqref{eq:last lr array DL} is particularly evident for WSD and cyclic schedules: for WSD schedule, we see a sudden decrease of loss matching the decay of learning rate; for cyclic schedule, we see the periodic oscillation of loss matching that of learning rate. In particular, we use half the iterations to fit (solid line) and half to predict (dashed line). Our loss prediction by \eqref{eq:last lr array DL} is tight after a short period of training, with $R^2$ score $\geq 0.95$, highlighting the strong applicability of \Cref{abs:first}.
\begin{figure}[H]
    \centering
    \hspace{-0.1cm}
    \begin{subfigure}[b]{0.20\textwidth} \includegraphics[width=\linewidth]{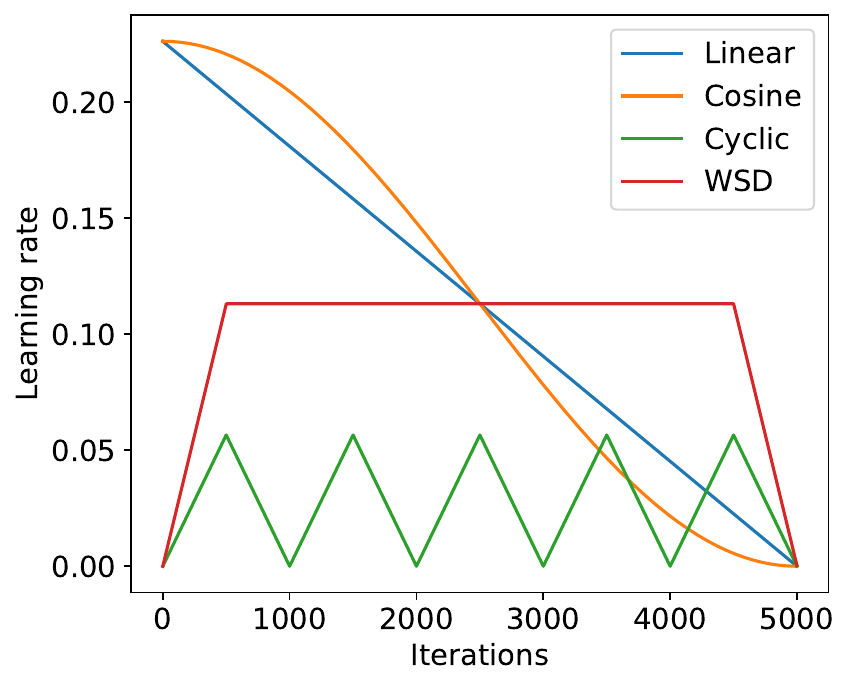}
    \caption{schedules}
    \end{subfigure}
    \hspace{-0.22cm}
    \begin{subfigure}[b]{0.20\textwidth}
    \includegraphics[width=\linewidth]{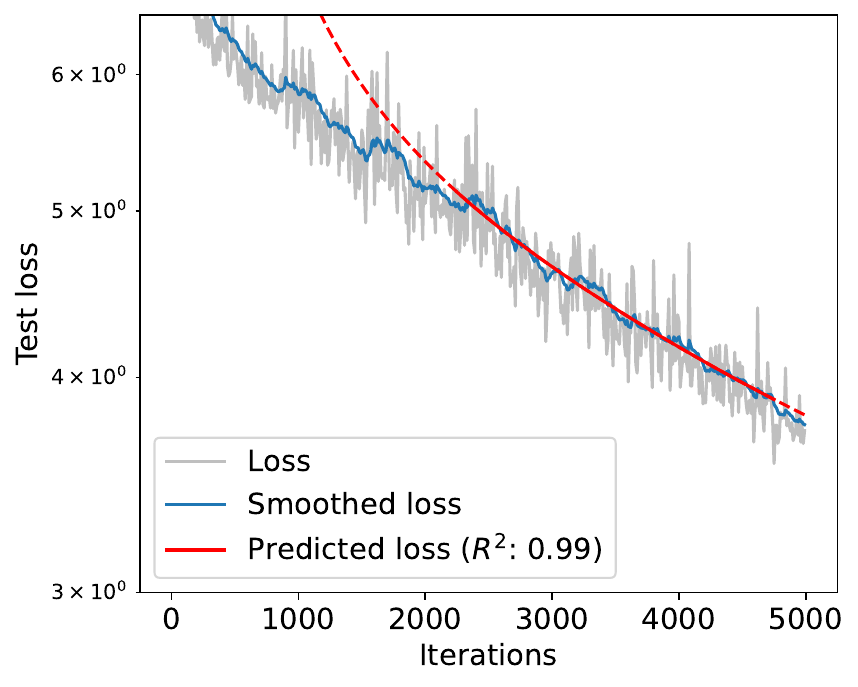}
    \caption{Linear}
    \end{subfigure}
    \hspace{-0.22cm}
    \begin{subfigure}[b]{0.20\textwidth}
\includegraphics[width=\linewidth]{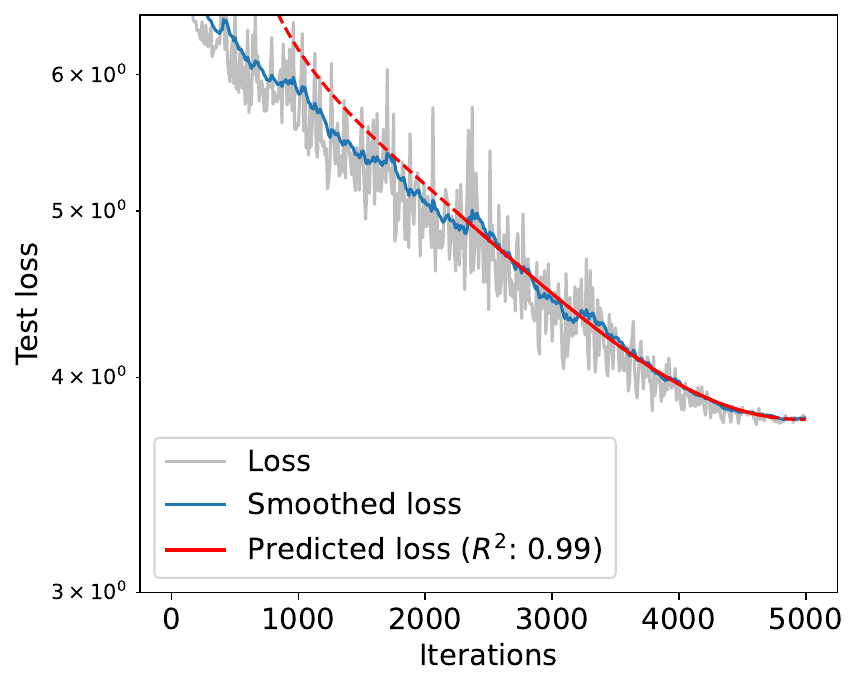}
    \caption{Cosine}
    \end{subfigure}
    \hspace{-0.22cm}
    \begin{subfigure}[b]{0.20\textwidth}
    \includegraphics[width=\linewidth]{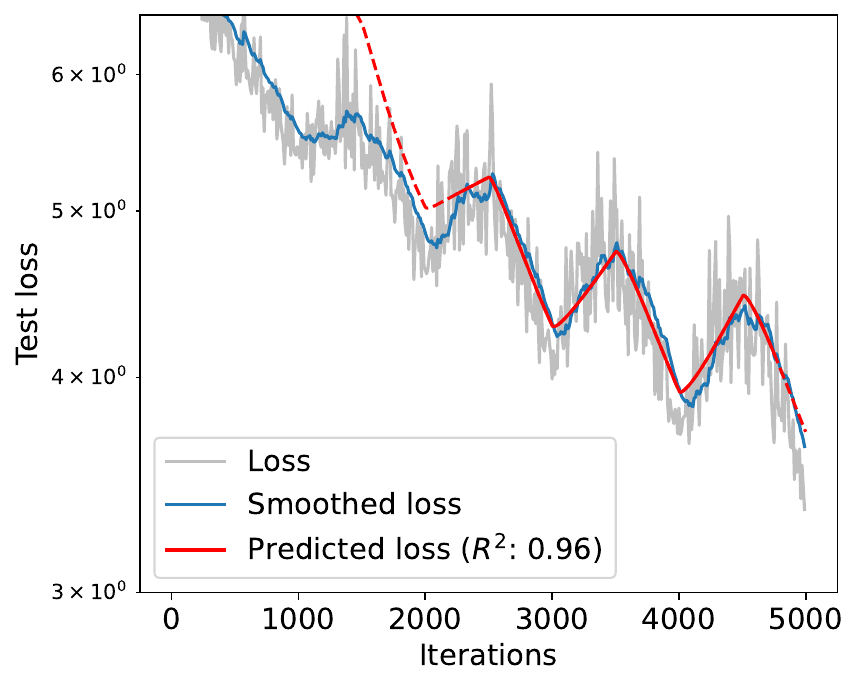}
    \caption{Cyclic}
    \end{subfigure}
    \hspace{-0.22cm}
    \begin{subfigure}[b]{0.20\textwidth}
    \includegraphics[width=\linewidth]{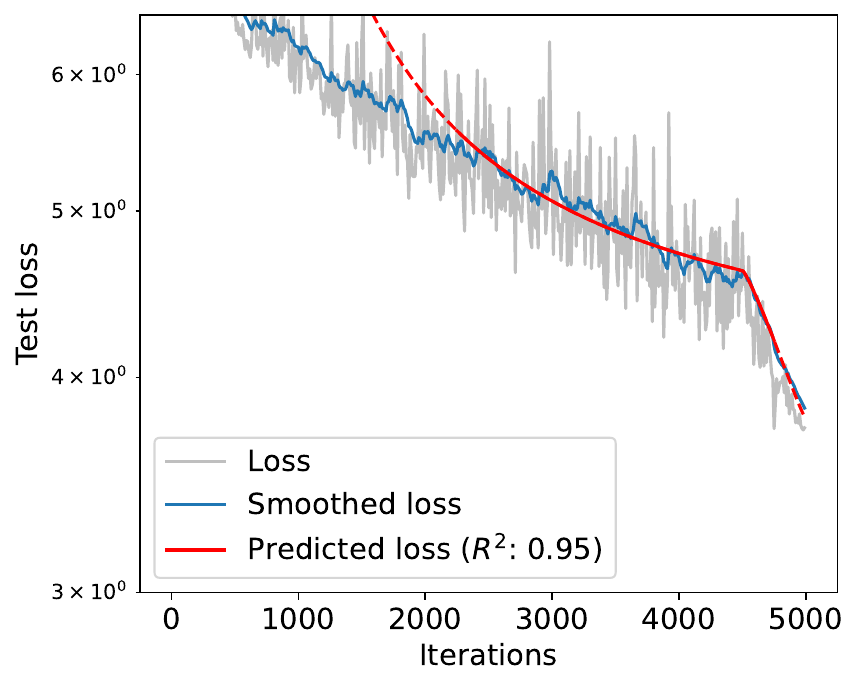}
    \caption{WSD}
    \end{subfigure}
    \caption{Sequence-to-sequence prediction by \eqref{eq:last lr array DL} for ResNet18 on ImageNet with SGD.}
    \label{fig:resnet-imagenet}
\end{figure}



\subsection{Generalizing to adaptive optimizers}
\label{sec:adamw}
We further test adaptive optimizers beyond SGD.
Firstly, we compare ResNet18 trained with AdamW in \Cref{fig:resnet-imagenet-adamw} to SGD in \Cref{fig:resnet-imagenet} and observe the same patterns that \eqref{eq:last lr array DL} holds with $R^2$ score $\geq 0.95$.



\begin{figure}[H]
    \centering
    \begin{subfigure}[b]{0.19\textwidth} \includegraphics[width=\linewidth]{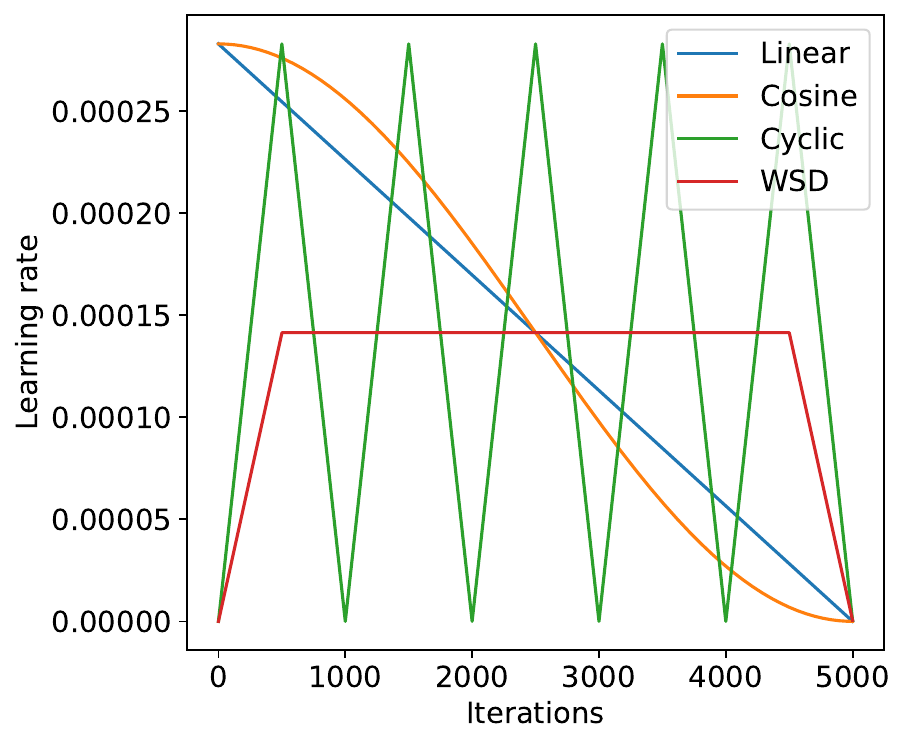}
    \caption{schedules}
    \label{fig:imagenet-lr-schedule}
    \end{subfigure}
    \hfill
    \begin{subfigure}[b]{0.19\textwidth}
    \includegraphics[width=\linewidth]{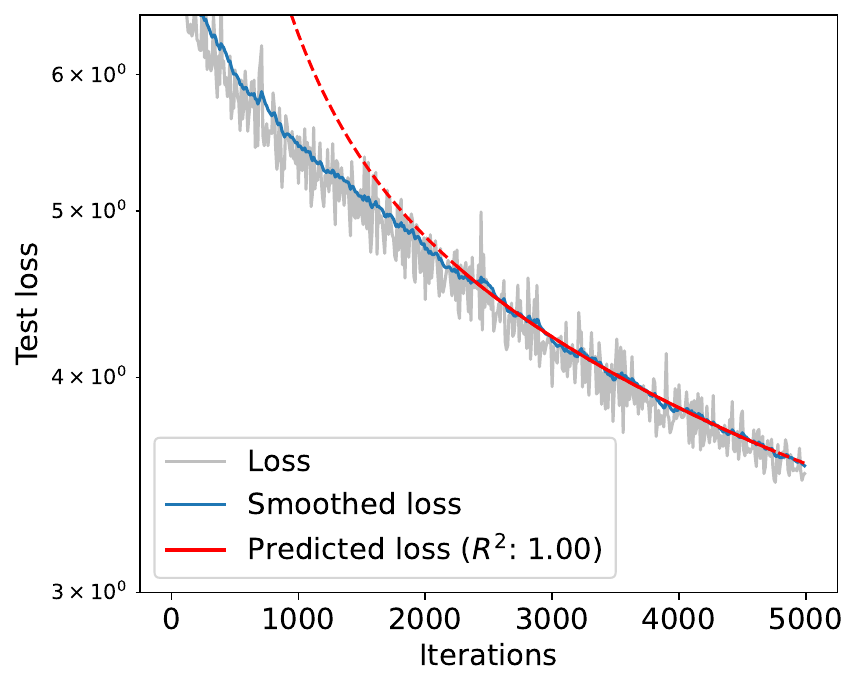}
    \caption{Linear}
    \label{fig:imagenet-linear}
    \end{subfigure}
    \hfill
    \begin{subfigure}[b]{0.19\textwidth}
\includegraphics[width=\linewidth]{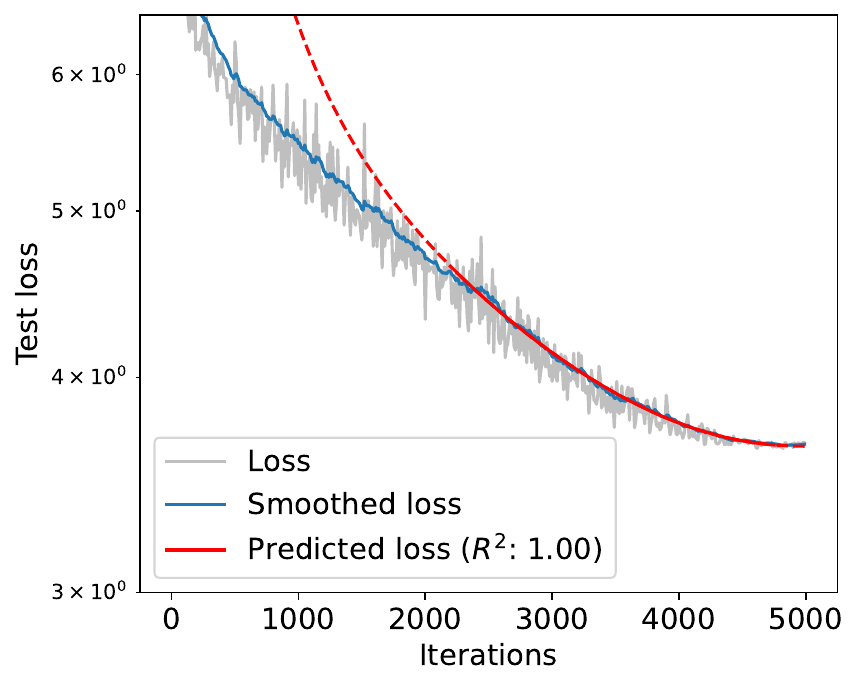}
    \caption{Cosine}
    \label{fig:imagenet-cosine}
    \end{subfigure}
    \hfill
    \begin{subfigure}[b]{0.19\textwidth}
    \includegraphics[width=\linewidth]{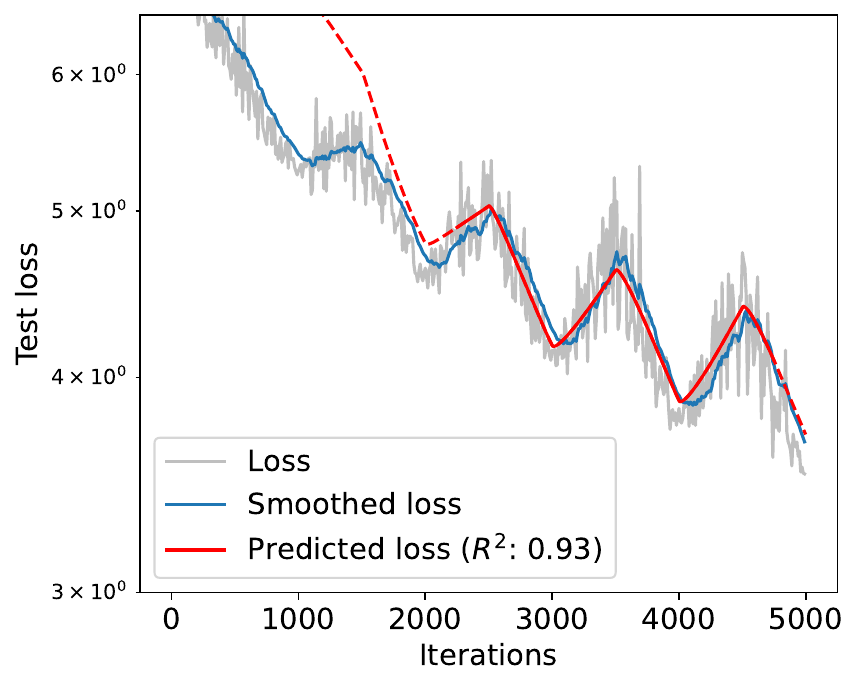}
    \caption{Cyclic}
    \label{fig:imagenet-cycle}
    \end{subfigure}
    \begin{subfigure}[b]{0.19\textwidth}
    \includegraphics[width=\linewidth]{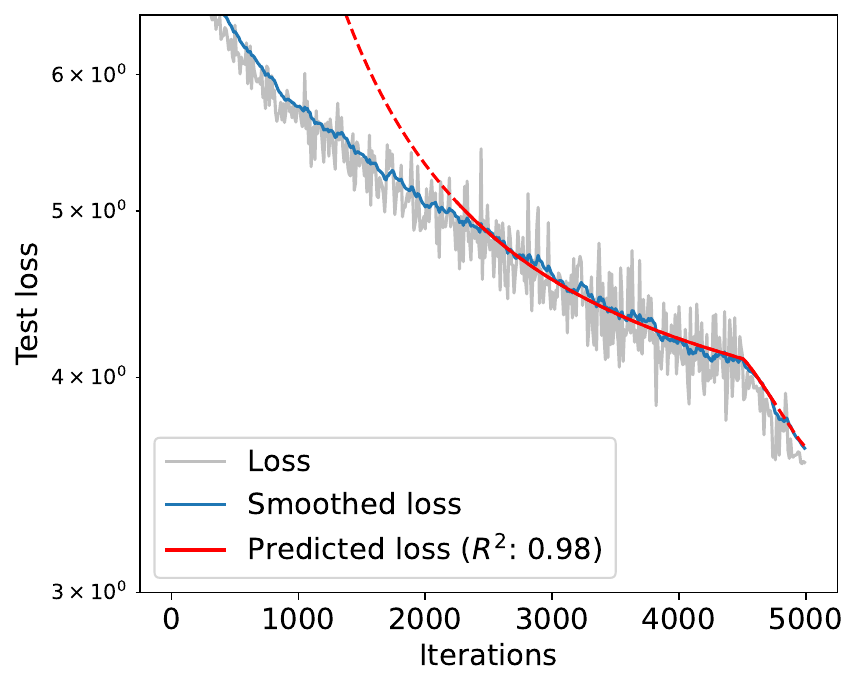}
        \caption{WSD}
    \label{fig:imagenet-wsd}
    \end{subfigure}
    \vspace{-0.2cm}
    \caption{Sequence-to-sequence prediction by \eqref{eq:last lr array DL} for ResNet18 on ImageNet with AdamW.}
    \label{fig:resnet-imagenet-adamw}
\end{figure}

We further pre-train GPT2 (124M) with AdamW and Muon-NSGD \citep{boreiko2025towards} optimizers, and observe consistent patterns in \Cref{fig:nanogpt-adamw} and \Cref{fig:nanogpt-muon}. The prediction $R^2$ scores are $\geq0.95$ in all cases.

\begin{figure}[!htb]
    \centering
    \begin{subfigure}[b]{0.19\textwidth} \includegraphics[width=\linewidth]{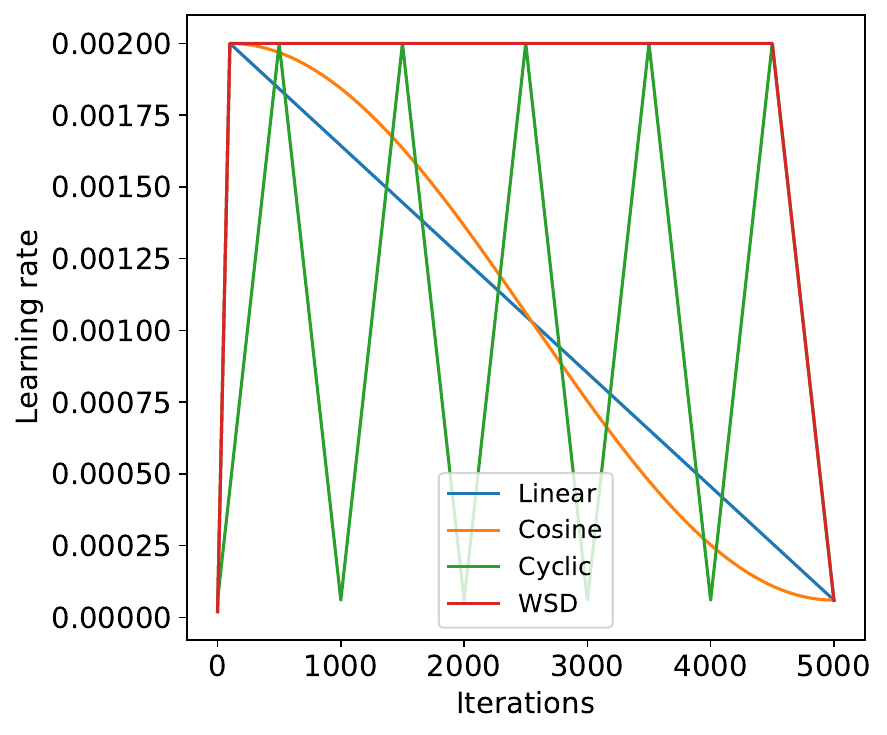}
    \caption{schedules}
    \label{fig:nanogpt-adamw-lr-schedule}
    \end{subfigure}
    \hfill
    \begin{subfigure}[b]{0.19\textwidth}
    \includegraphics[width=\linewidth]{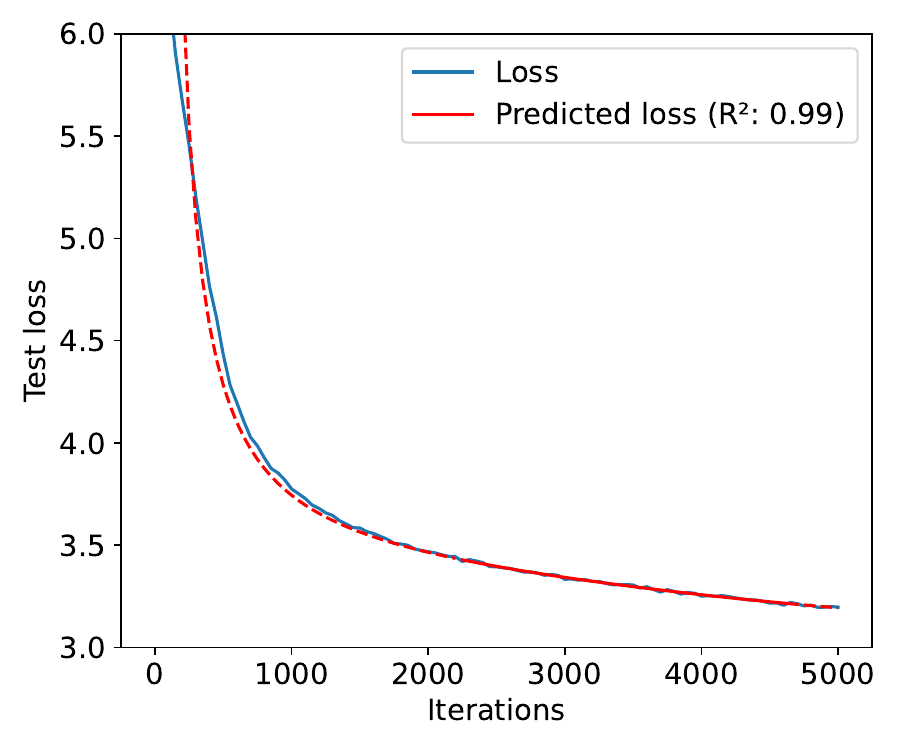}
    \caption{Linear}
    \label{fig:nanogpt-adamw-linear}
    \end{subfigure}
    \hfill
    \begin{subfigure}[b]{0.19\textwidth}
\includegraphics[width=\linewidth]{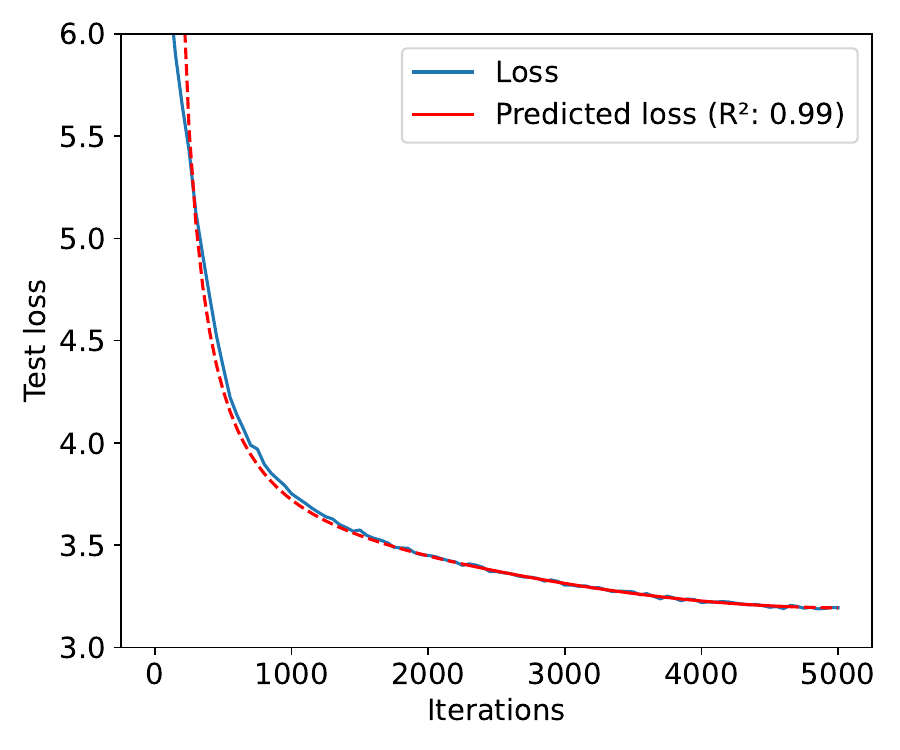}
    \caption{Cosine}
    \label{fig:nanogpt-adamw-cosine}
    \end{subfigure}
    \hfill
    \begin{subfigure}[b]{0.19\textwidth}
    \includegraphics[width=\linewidth]{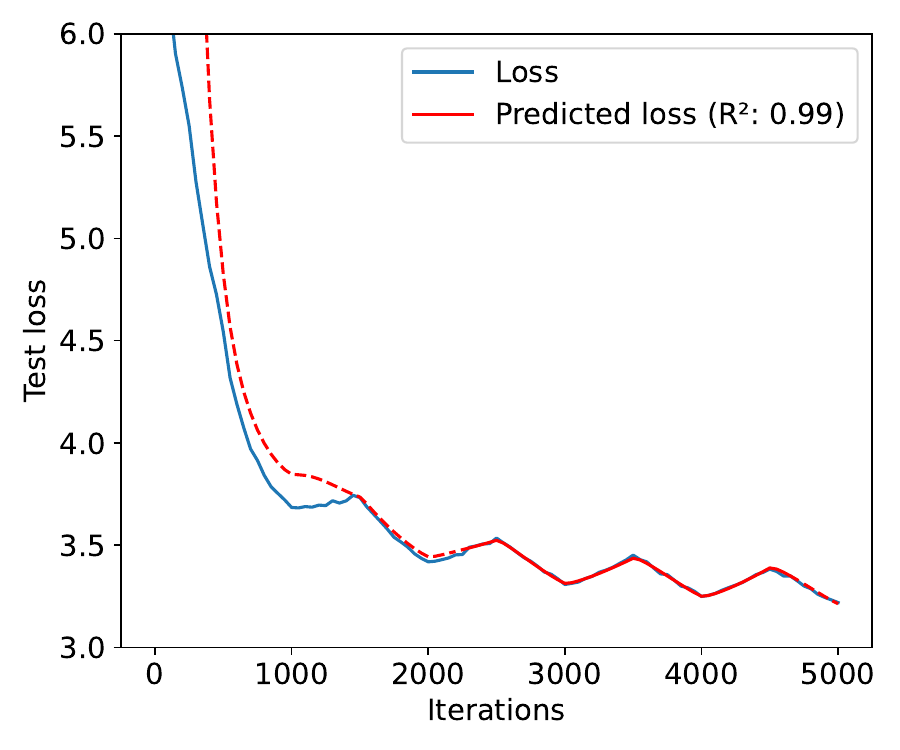}
    \caption{Cyclic}
    \label{fig:nanogpt-adamw-cycle}
    \end{subfigure}
    \begin{subfigure}[b]{0.19\textwidth}
    \includegraphics[width=\linewidth]{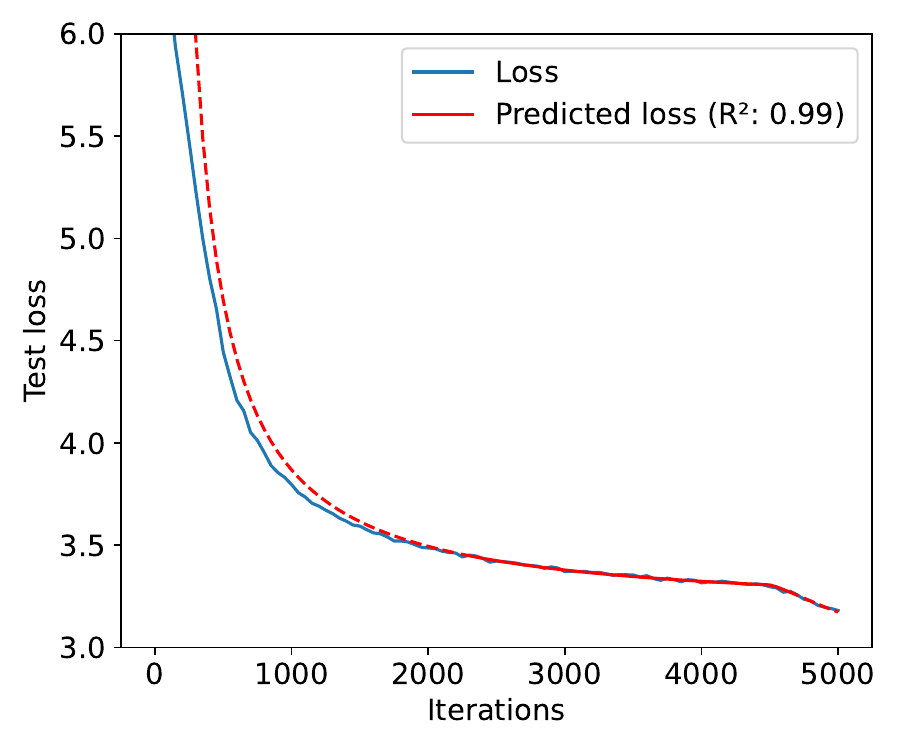}
        \caption{WSD}
    \label{fig:nanogpt-adamw-wsd}
    \end{subfigure}
    \vspace{-0.2cm}
    \caption{Sequence-to-sequence prediction by \eqref{eq:last lr array DL} for GPT2 on OpenWebText with AdamW.}

    \label{fig:nanogpt-adamw}
\end{figure}

\begin{figure}[!htb]
    \centering
    \begin{subfigure}[b]{0.19\textwidth} \includegraphics[width=\linewidth]{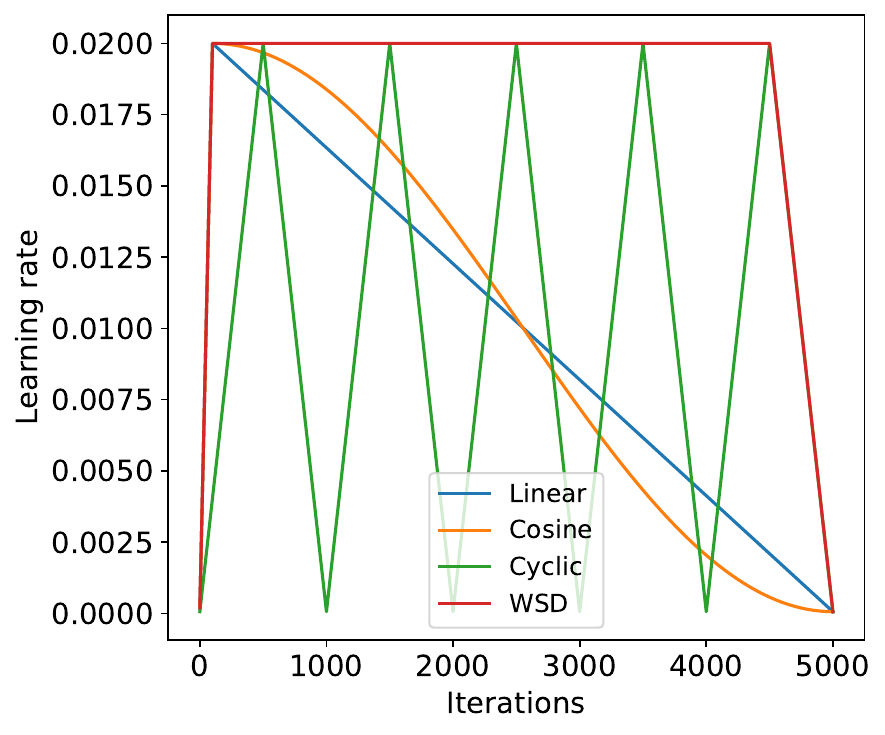}
    \caption{schedules}
    \label{fig:nanogpt-muon-lr-schedule}
    \end{subfigure}
    \hfill
    \begin{subfigure}[b]{0.19\textwidth}
    \includegraphics[width=\linewidth]{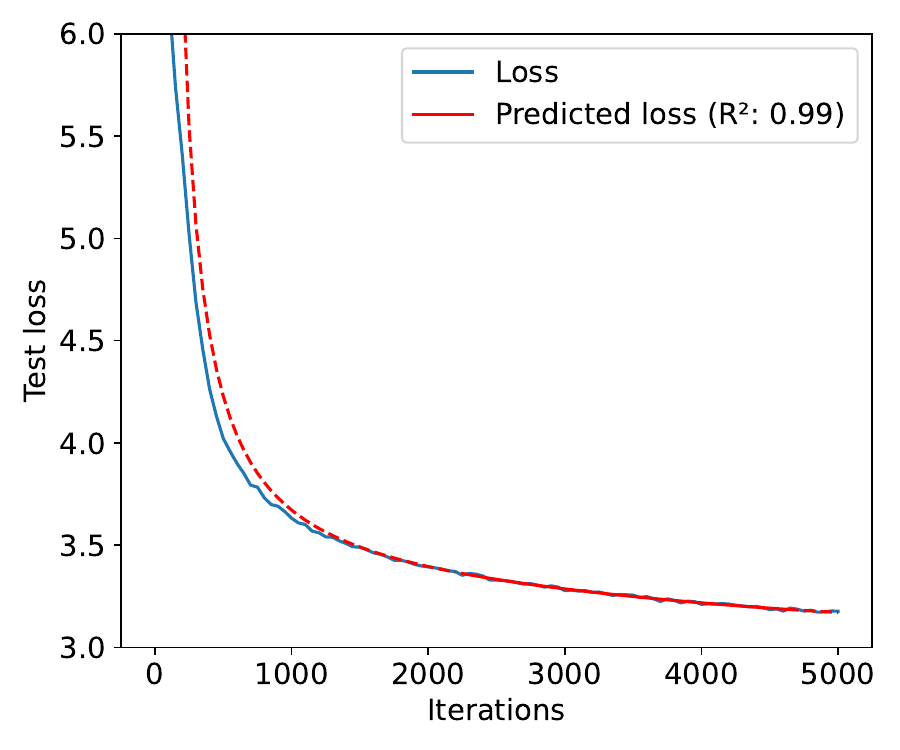}
    \caption{Linear}
    \label{fig:nanogpt-muon-linear}
    \end{subfigure}
    \hfill
    \begin{subfigure}[b]{0.19\textwidth}
\includegraphics[width=\linewidth]{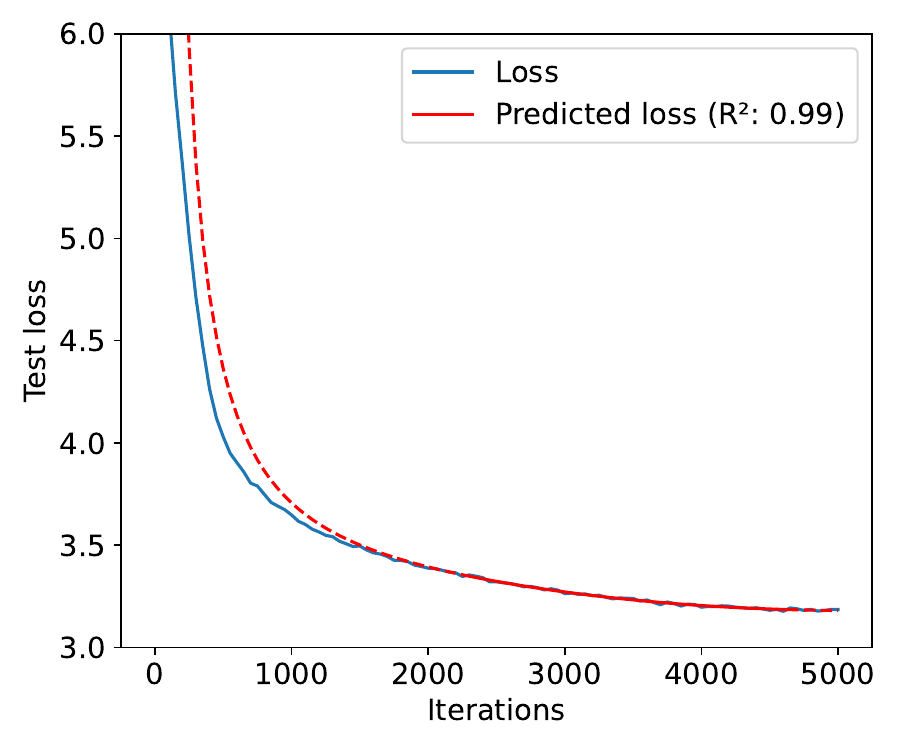}
    \caption{Cosine}
    \label{fig:nanogpt-muon-cosine}
    \end{subfigure}
    \hfill
    \begin{subfigure}[b]{0.19\textwidth}
    \includegraphics[width=\linewidth]{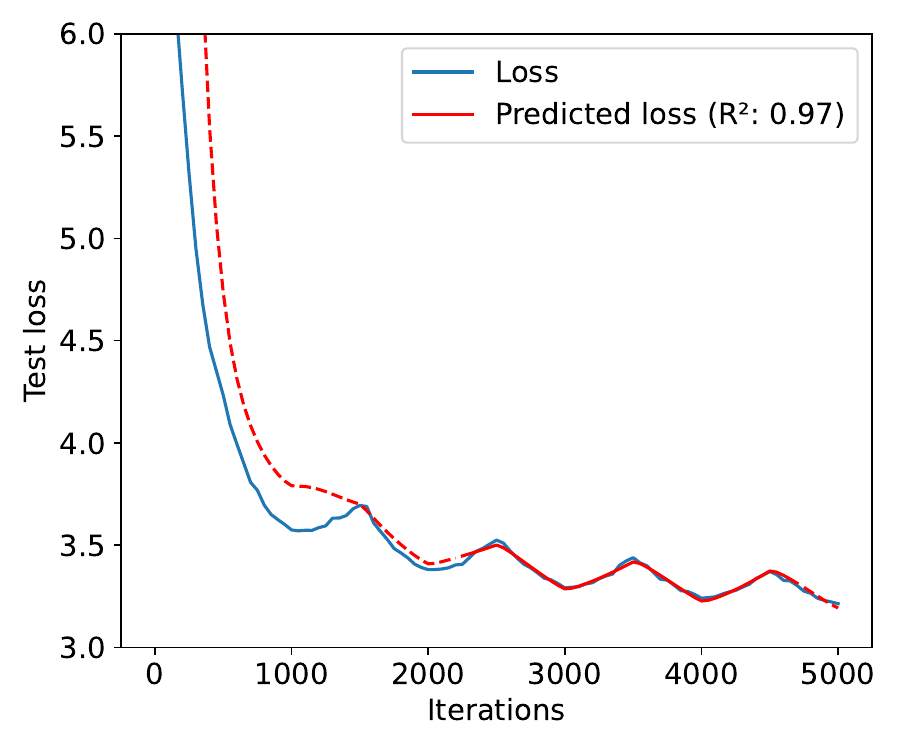}
    \caption{Cyclic}
    \label{fig:nanogpt-muon-cycle}
    \end{subfigure}
    \begin{subfigure}[b]{0.19\textwidth}
    \includegraphics[width=\linewidth]{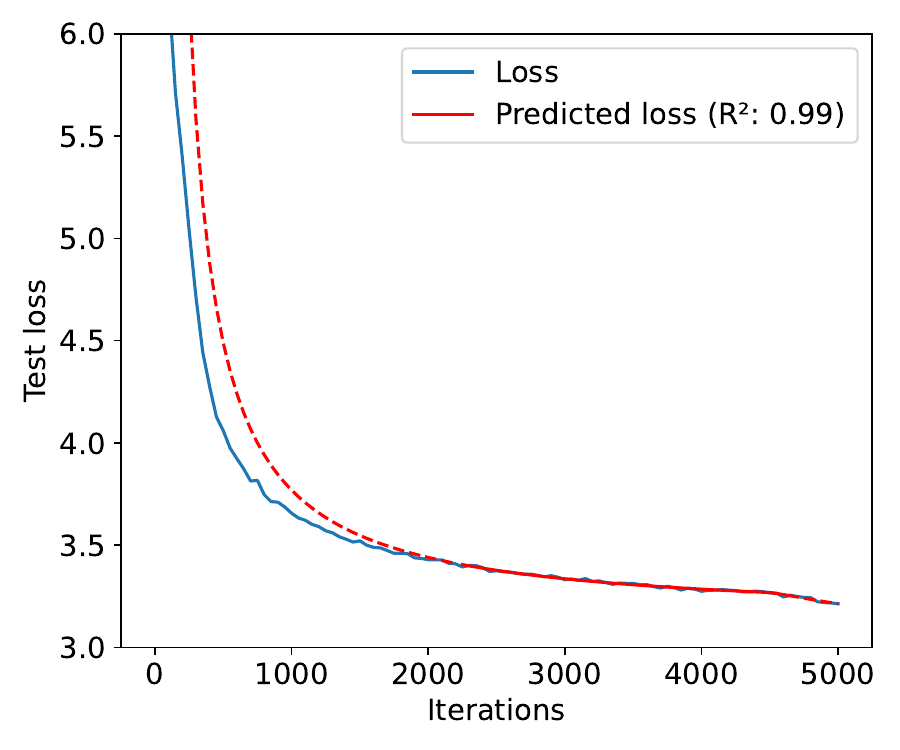}
        \caption{WSD}
    \label{fig:nanogpt-muon-wsd}
    \end{subfigure}
    \vspace{-0.2cm}
    \caption{Sequence-to-sequence prediction by \eqref{eq:last lr array DL} for GPT2 on OpenWebText with Muon-NSGD.}

    \label{fig:nanogpt-muon}
\end{figure}
\vspace{-0.3cm}

\section{Loss characterization at last iteration}
\label{sec:4}


Towards building a scaling law of loss and learning rate, we focus on the last iterate $\w_T$ and further abstract the loss characterization in \eqref{eq:last lr array DL}. 

\begin{abstraction}[generalized from \eqref{eq:peak and loss}]
\label{abs:dl last lr}
For general optimizers under deep learning and for a qualified learning rate schedule with any peak learning rate  $\etap$, we have 
\begin{align}
\E L(\w_T)\sim  \tilde L_\infty+\frac{\tilde q_1^2}{T\etap}+\etap \tilde q_2^2 := L_\text{DL-last}(\etap,T)
\label{eq:peak and loss DL}
\end{align}
\end{abstraction}
\vspace{-0.3cm}
In contrast to \Cref{abs:first}, we require the following modifications:
\begin{itemize}
\item We restrict our analysis to the qualified schedules that satisfy \Cref{def:qualifying}, such as linear and cosine decay, instead of any schedules.
\item We renounce the exact forms of $\tilde q_1$ and $\tilde q_2$ (e.g. $\tilde q_1=\tilde D, \tilde q_2=\tilde G$ for linear decay; $\tilde q_2=\sqrt{1.061}\tilde G$ for cosine decay; $\tilde q_1=\tilde D/(1+c)$ for WSD) and will derive these quantities in a data-driven manner.
\item We focus on $\w_T$ rather than any iterate $\w_\tau$ and transform the inequality \eqref{eq:last lr array DL} to an approximate equality for moderately large $T$.
\end{itemize}

\subsection{Loss under any peak learning rate}

We test \eqref{eq:peak and loss DL} using losses reported in \citep{li2025predictable}. These results were obtained under optimally tuned hyperparameters, where the language models (dense and Mixture-of-Experts; MoE) were trained on a mixture of web text, mathematics, and code. The training used AdamW optimizer with cosine decaying schedule (see details in Section 3.2 and Appendix A of \citep{li2025predictable}). 

We observe high $R^2$ scores ($\geq$ 0.95) across different model sizes as well as model architectures, which empirically validates \Cref{abs:dl last lr}.

\begin{figure}[!htb]
    \centering
    \includegraphics[width=.3\linewidth]{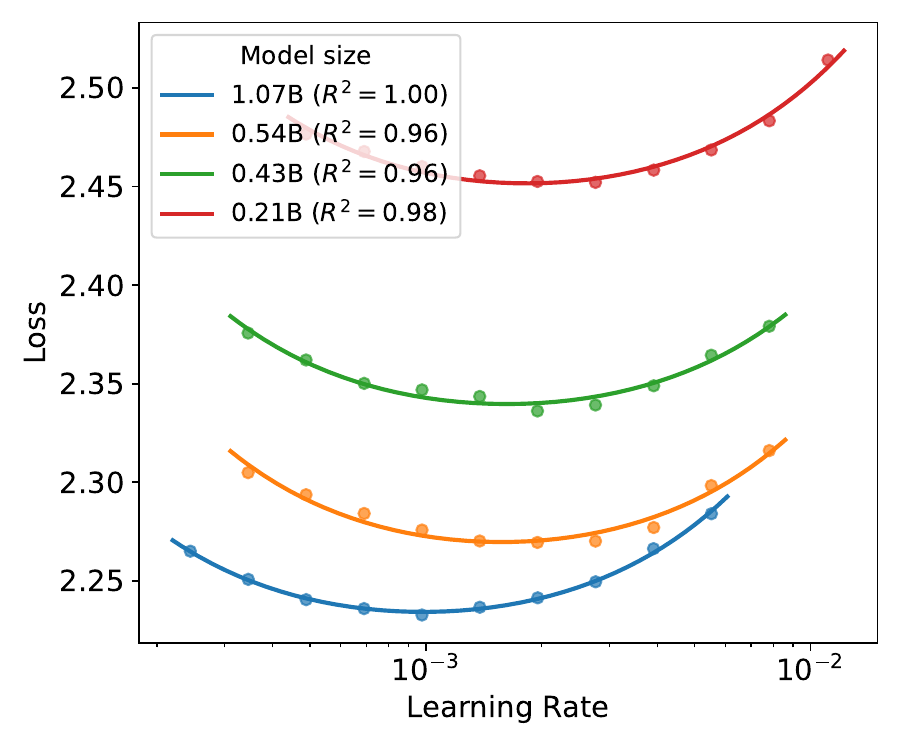}
    \includegraphics[width=.3\linewidth]{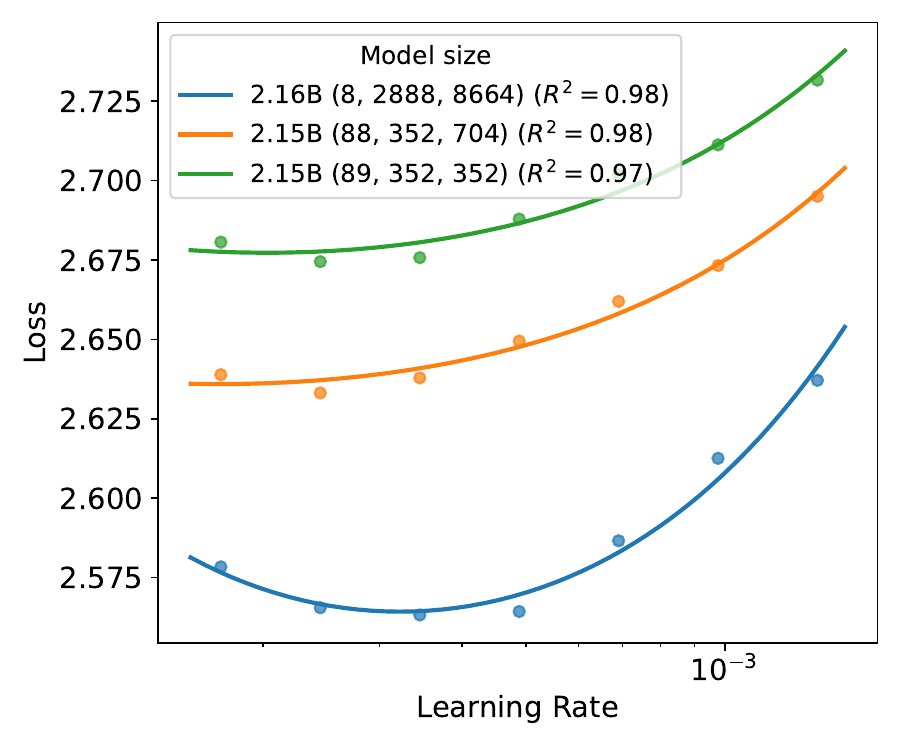}
\vspace{-0.2cm}
\caption{Loss versus learning rate for dense models (left) and MoE models (right) by \citep{li2025predictable}. In the right plot, the orange and green curves correspond to different architectures despite having the same model size. Specific model configurations are summarized in \Cref{app:hitch-hiker}.}
    \label{fig:hitch-hiker}
\end{figure}




\subsection{Loss under optimal peak learning rate}
We further derive and test \eqref{eq:peak and loss DL} under the optimal $\etap$ in \Cref{abs:dl last optimal}.
\begin{abstraction}[generalized from \Cref{col:sgd ref}, part 2]
\label{abs:dl last optimal}
For general optimizers under deep learning and for a qualified learning rate schedule with optimal peak learning rate $\etap$, we have $\E L(\w_T)\sim \tilde L_\infty+2\tilde q_1\tilde q_2/\sqrt{T}$.
\end{abstraction}


To validate \Cref{abs:dl last optimal}, we scrutinize the loss values and training horizons in \citep{hoffmann2022training} for various models from $0.074B$ to $12.56B$, which were used to establish the compute-optimal scaling law for large language models. These runs trained Chinchilla models (same as Gopher \citep{rae2021scaling}) with AdamW under cosine decaying schedule. The original loss values are not publicized in \citep{hoffmann2022training} but they are reconstructed by \citep{besiroglu2024chinchilla}. Note that \cite{hoffmann2022training} only gives FLOPs\footnote{FLOP$\approx 6\times\text{token size}\times\text{model size}=6\times\text{batch size}\times\text{iterations}\times\text{model size}$. The batch size is not released in this work and assumed to be constant across all runs.}, rather than the training iterations. Therefore, we translate the training horizon to token size$=$FLOPs$/6/$model size$=\text{batch size}\times T$. 

\begin{minipage}[b]{0.6\textwidth}
\centering
\captionof{table}{Summary of linear regression on $\E L(\w_T)$ and $1/\sqrt{T\cdot\text{batch size}}$ from \citep{hoffmann2022training,besiroglu2024chinchilla}. Full table is deferred to \Cref{tab:replica all}. Runs for 2B model is visualized in \Cref{fig:vis hitchhiker}. }
\resizebox{\linewidth}{!}{
\begin{tabular}{c|c|c|c|c}
  model size(B) &num of horizons  &$2\tilde q_1\tilde q_2$&$\tilde L_\infty$&$R^2$ score\\\hline
0.074&5&3.22e+04&2.825&0.991\\
0.140&7&3.04e+04&2.670&0.991\\
0.279&8&3.29e+04&2.498&0.999\\
0.425&8&3.27e+04&2.430&0.998\\
0.632&8&3.17e+04&2.367&0.998\\
1.143&10&3.10e+04&2.275&0.998\\
1.429&9&3.18e+04&2.253&0.996\\
1.611&9&3.36e+04&2.228&0.995\\
\textcolor{blue}{2.004}&\textcolor{blue}{8}&\textcolor{blue}{3.62e+04}&\textcolor{blue}{2.178}&\textcolor{blue}{0.999}\\
2.280&7&4.41e+04&2.128&1.000\\
2.979&10&5.90e+04&2.016&0.990\\
4.519&6&3.83e+04&2.106&0.978\\
6.792&8&4.66e+04&2.023&0.999\\
9.290&4&4.29e+04&2.046&0.988\\
12.56&3&4.23e+04&2.053&1.000\\
\end{tabular}
}
\label{tab:replica}
\end{minipage}
\hfill
\begin{minipage}[b]{0.36\textwidth}
\centering
\includegraphics[width=\linewidth]{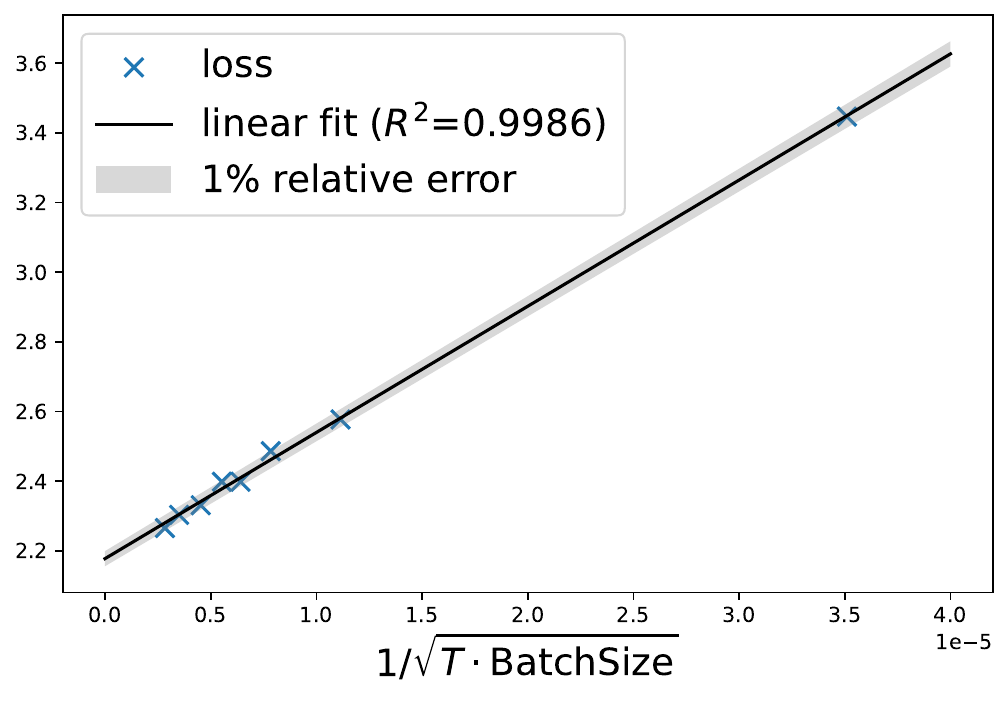}
\includegraphics[width=\linewidth]{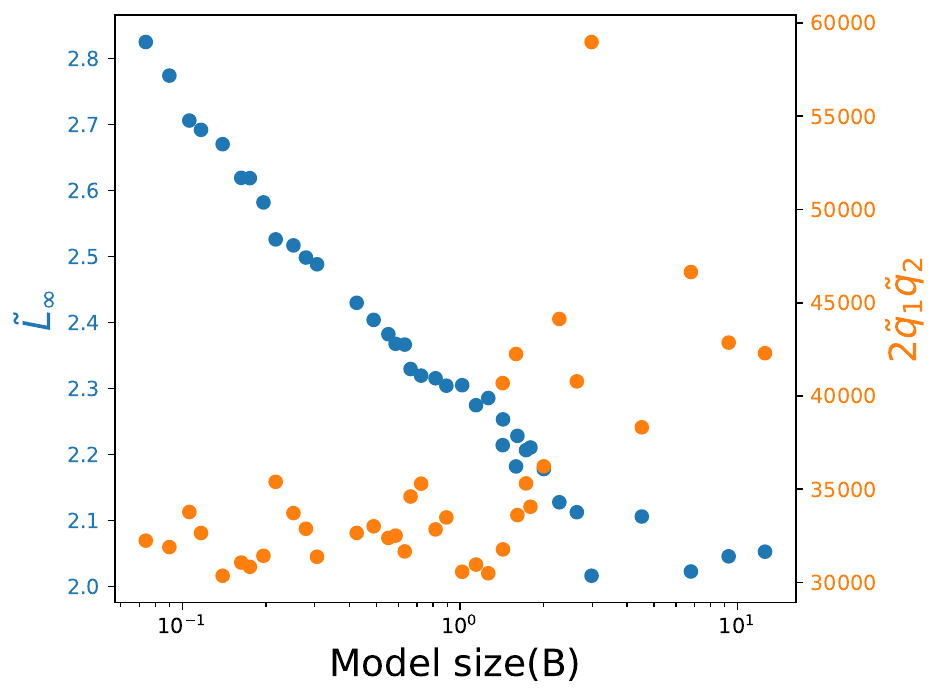}
\captionof{figure}{Upper: 2B model precisely fits \Cref{abs:dl last optimal} across a wide range of training horizons. Lower: As model size increases, $\tilde L_\infty$ log-linearly decreases and $2\tilde q_1\tilde q_2$ roughly increases.}
\label{fig:vis hitchhiker}
\end{minipage}

In \Cref{tab:replica}, we consistently observe precise fitting of \Cref{abs:dl last optimal} with $R^2\geq 0.978$, up to FLOPs=1e22 and over 300B tokens. We note that all loss values are within 1\% relative error of our prediction in \Cref{fig:vis hitchhiker}.

\section{Two-dimensional scaling law for learning rate}
\label{sec:experiments}

The success of transferring insights from convex analysis to deep learning underscores the appeal of predicting and controlling the loss via learning rate. In this section, we propose two-dimensional scaling law of optimal loss and learning rate at various model sizes $(N)$ and training horizons $(T)$.


Notably, we must take one step further from \Cref{abs:dl last optimal}, because we do not know the optimal $\etap$, except it has a form $\frac{\tilde q_1}{\tilde q_2\sqrt{T}}$ with unknown $\tilde q_1$ and $\tilde q_2$.

\begin{abstraction}[generalized from \Cref{col:sgd ref}, part 1]
\label{abstract:4}
For general optimizers under deep learning and for a qualified learning rate schedule with scaled peak learning rate $\etap={\etar}/{\sqrt{T}}$, we have 
$$\E L(\w_T)\sim \tilde L_\infty + \tilde Q(\etar)/\sqrt{T}, \forall\etar.$$
\end{abstraction}
In words, we have $O(1/\sqrt{T})$ loss convergence under $1/\sqrt{T}$-scaled $\etap$, including but not limited to the optimal $\etar$ in \Cref{abs:dl last optimal}. Notice this expression is linear in $1/\sqrt{T}$ and can be visualized by a straight line. We seek the optimal $\etar$ by multiple small-scale runs and we estimate $\tilde L_\infty$ and $\tilde Q$ via linear regression in \Cref{sec:across T} and \Cref{sec:across N}.

Finally, we present a scaling law that predicts both loss and optimal learning rate:
\begin{align}
\begin{split}
\E L(N,T)\sim \tilde L_\infty(\etar^*;N)+\tilde Q(\etar^*;N)/\sqrt{T}
\\
\etap^*(N,T)\sim \etap^*(N_\text{small},T_\text{small})/\sqrt{T/T_\text{small}}
\end{split}
\label{eq:loss scaling law}
\end{align}
where
$N_\text{small}$ is a smaller model with $N_\text{small}\leq N$, and $T_\text{small}$ is a shorter training horizon with $T_\text{small}\leq T$, so that we can use small-scale training to inform the large-scale training's hyperparameter.

\subsection{Experiment settings}
We train GPT2 language models on OpenWebText for various training horizons, following nanoGPT codebase \citep{karpathy_nanogpt}. We apply two optimizers: AdamW and Muon-NSGD, both with 0.01 weight decay and without gradient clipping. Here Muon-NSGD is adapted from the original Muon by (1) optimizing all 2D tensors with Muon and other tensors with normalized SGD, i.e. NSGD, and (2) using a single learning rate for Muon and NSGD. We use the cosine decaying learning rate schedule that decays to 0 with 2\% warm-up. We conduct 240 runs, which are 4 settings (0.1B/AdamW, 0.1B/Muon-NSGD, 1B/Muon-NSGD, and 7B/Muon-NSGD), 10 training horizons from 100 to 500k iterations, and 6 $\etar$ from 0.01 to 30.0). See more details in \Cref{app:experiment}.

To illustrate the generality of our approach, we conduct \textit{additional experiments} in \Cref{app:overfitting} for
parameter-efficient training (LoRA \citep{hu2022lora}) and ablations over key hyperparameters, such as weight decay, gradient clipping, momentum coefficient, batch size, and random seeds. We consistently observe $O(1/\sqrt{T})$ loss convergence with scaled learning rate across different regimes. 

\subsection{Scaling across training horizons}
\label{sec:across T}
We train GPT2 (0.1B) for $T$ ranging from 100 to 500k steps. We obtain good fitting of linear regression over $1/\sqrt{T}<0.02$, i.e. $T>2.5k$ which indicates that \Cref{abstract:4} becomes predictive very early during training. This observation is consistent for two optimizers -- AdamW in \Cref{fig:GPT01B adamw} and Muon-NSGD in \Cref{fig:GPT01B muon}.

From the intercepts of fitted lines, we determine leverage the optimal $\etap^*=\etar^*/\sqrt{T}$ where $\etar^*=10$ for Muon-NSGD and $\etar^*=0.3$ for AdamW. As we observe from the zoom-in plots in \Cref{fig:GPT01B muon} and later in \Cref{fig:all 4 lines}, the extrapolation of our scaling law across training horizons is $\approx 80\times$.

\begin{figure}[!htb]
    \centering
    \vspace{-0.3cm}
\includegraphics[width=0.35\linewidth]{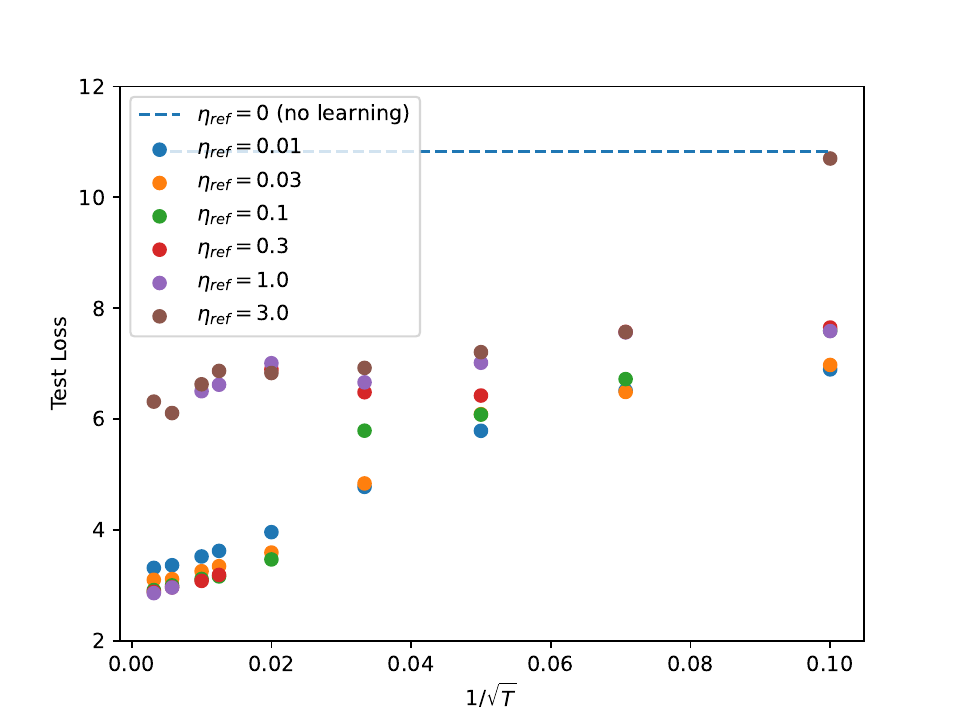}
\hspace{-0.65cm}
\includegraphics[width=0.35\linewidth]{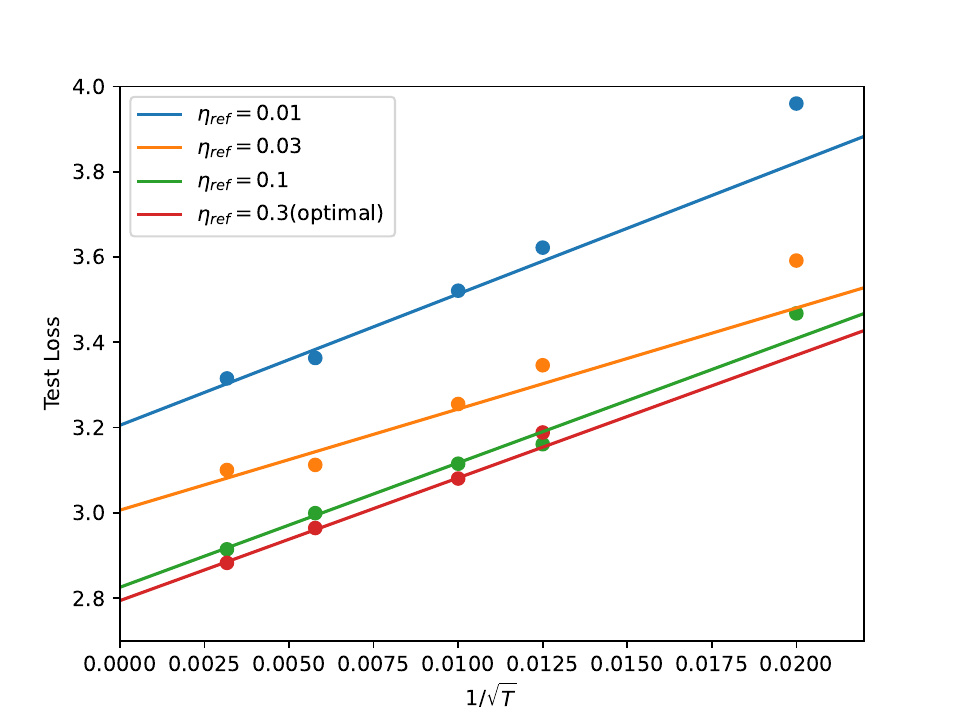}
\hspace{-0.65cm}
\includegraphics[width=0.35\linewidth]{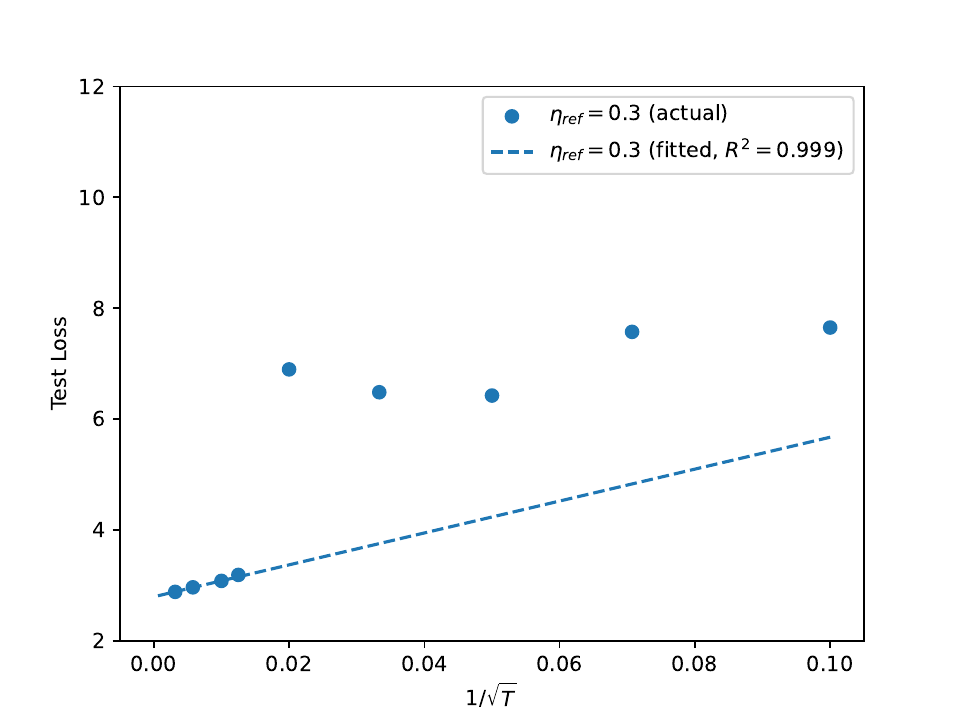} 
    \vspace{-0.3cm}
    \caption{Loss values (dots) and $1/\sqrt{T}$ prediction for GPT2 (0.1B) with AdamW.}
    \label{fig:GPT01B adamw}
\end{figure}

\begin{figure}[!htb]
    \centering
    \vspace{-0.3cm}
\includegraphics[width=0.35\linewidth]{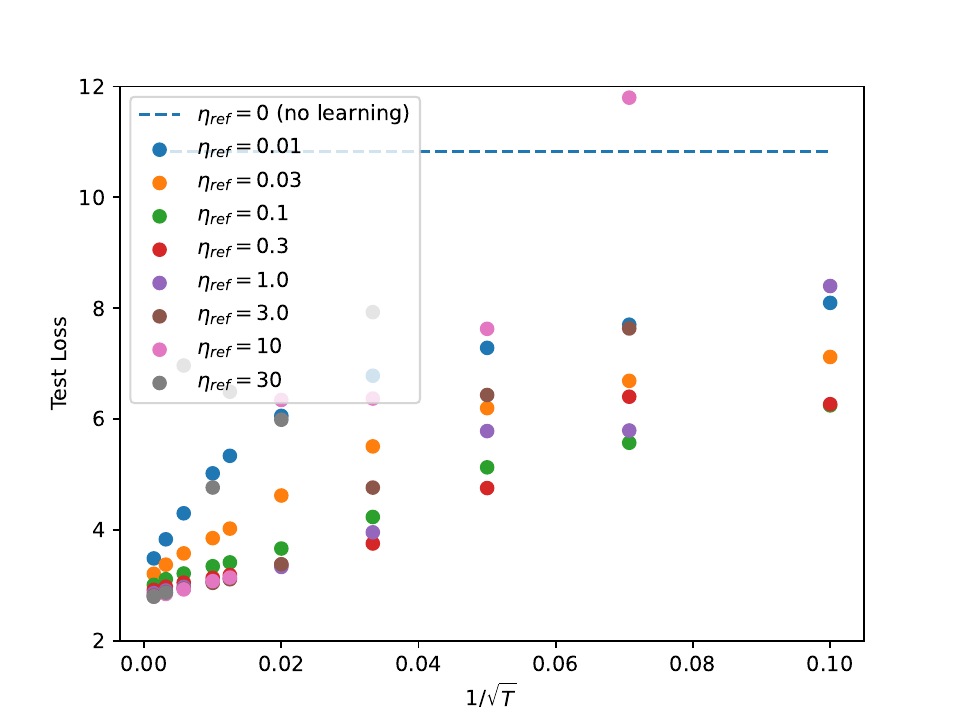}
    \hspace{-0.65cm}
\includegraphics[width=0.35\linewidth]{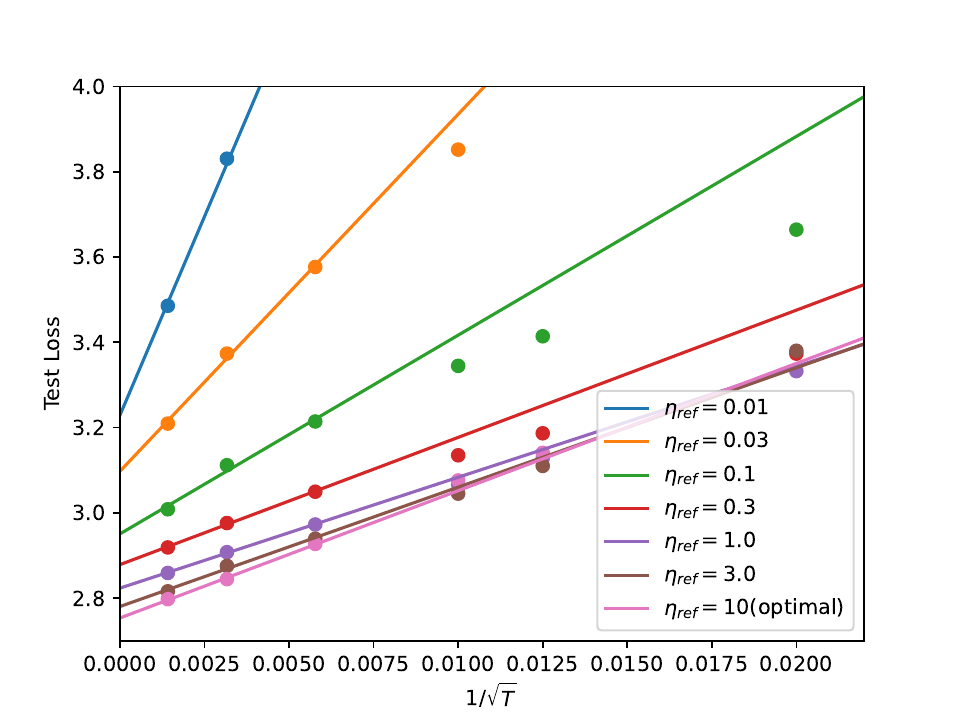}
    \hspace{-0.65cm}
\includegraphics[width=0.35\linewidth]{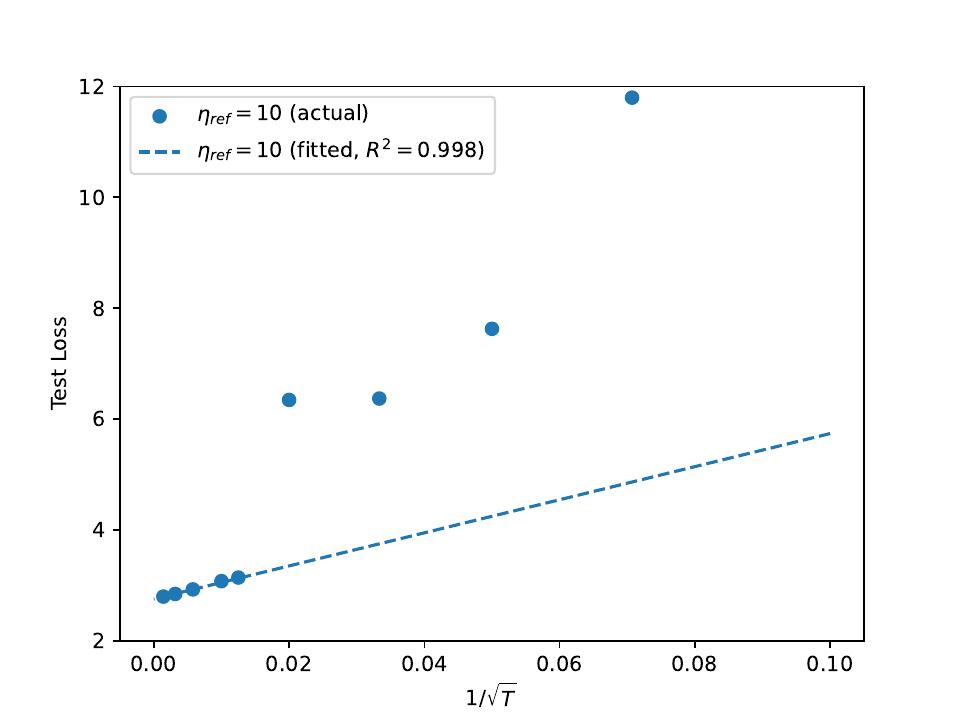}
    \vspace{-0.3cm}
    \caption{Loss values (dots) and $1/\sqrt{T}$ prediction for GPT2 (0.1B) with Muon-NSGD.}
    \label{fig:GPT01B muon}
\end{figure}

\subsection{Scaling across training horizons and model sizes}
\label{sec:across N}
We further train GPT2 1B and 7B models with Muon-NSGD, using the same $\etar^*$ that we transfer from GPT 0.1B model by \eqref{eq:loss scaling law}. Taking a closer look at 7B model losses, we observe precise fitting of our scaling law, extrapolating across model sizes by $70\times$.

\begin{figure}[!htb]
    \centering
    \vspace{-0.3cm}
\includegraphics[width=0.338\linewidth]{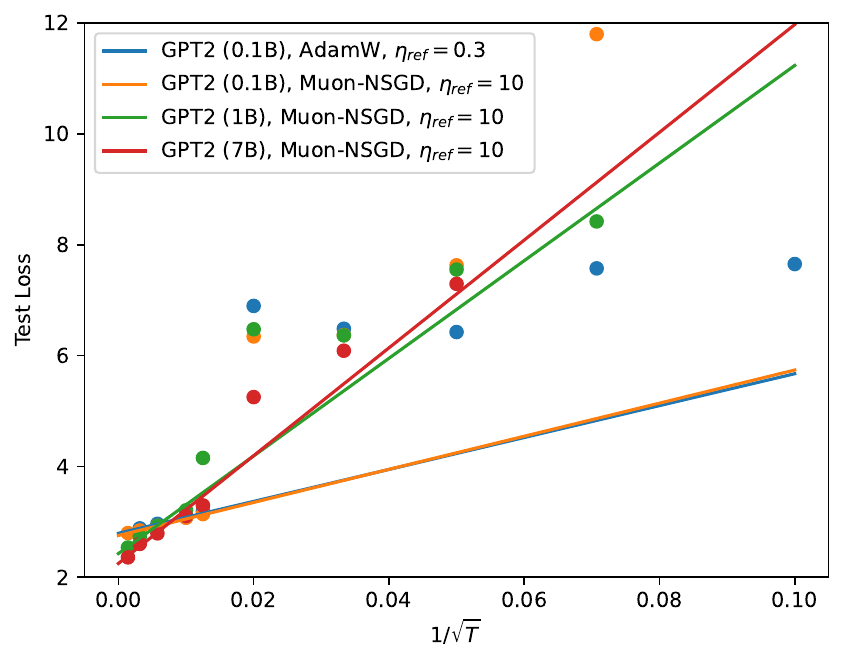}
    \hspace{-0.3cm}
\includegraphics[width=0.338\linewidth]{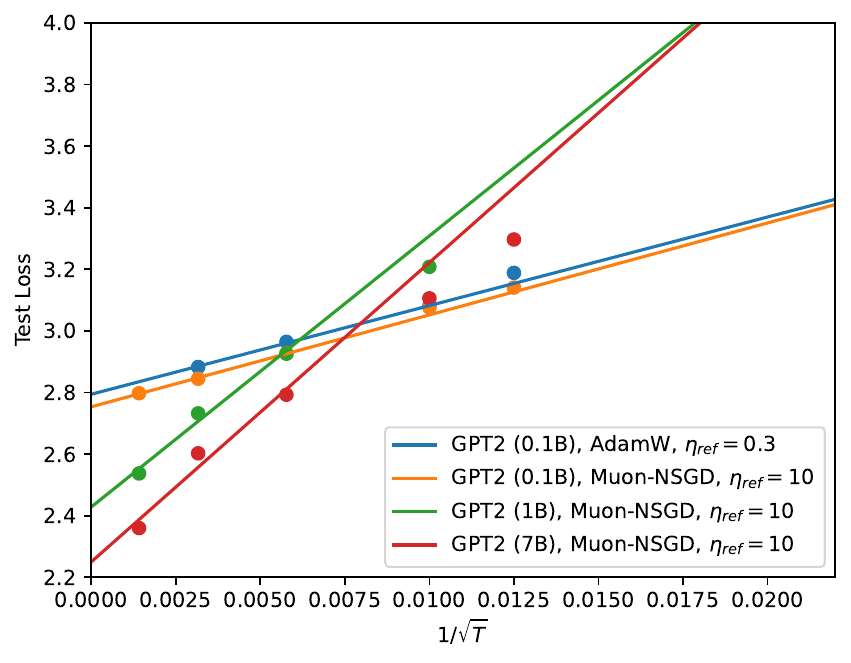}
    \hspace{-0.3cm}
\includegraphics[width=0.325\linewidth]{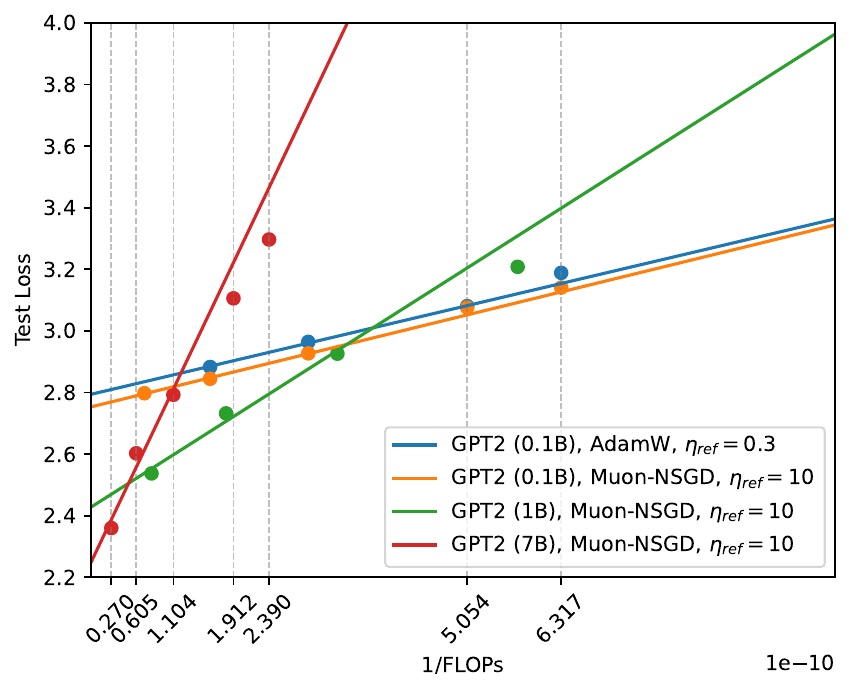}
    \vspace{-0.3cm}
    \caption{Loss values (dots) and $1/\sqrt{T}$ or $1/\sqrt{\text{FLOPs}}$ prediction for 4 settings.}
    \label{fig:all 4 lines}
\vspace{-0.2cm}
\end{figure}

\subsection{Scaling on multi-modal models}
We finetune vision-language models (VLM) with $\approx$1B parameters on the Cauldron dataset \citep{laurençon2024matters}, following nanoVLM codebase \citep{wiedmann2025nanovlm}. We apply AdamW optimizer with cosine decaying learning rate and 3\% warmup. We test three values of $\etar$ across training horizons up to 14.4k iterations. In \Cref{fig:multimodal}, we observe good fitting when $T>2000$, thus generalizing our scaling law to multi-modal models. We note that VLM has multiple components (language backbone, vision backbone, and modality projector) and each has a separate learning rate, all of which are $1/\sqrt{T}$ scaled.

\begin{figure}[!htb]
\vspace{-0.5cm}
    \centering
    \includegraphics[width=0.4\linewidth]{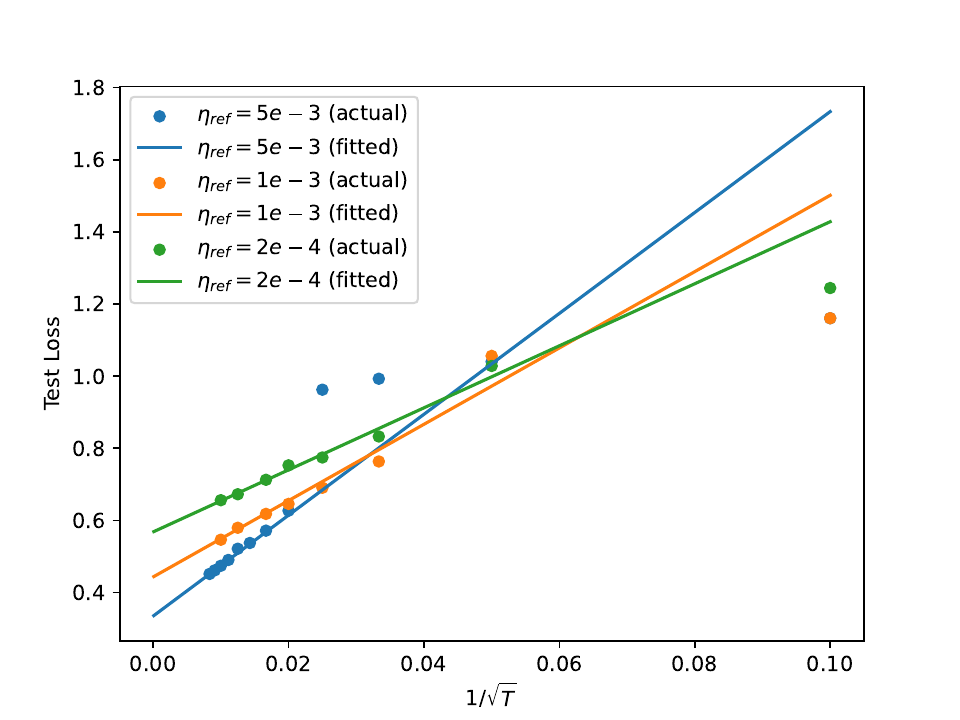}
    \includegraphics[width=0.4\linewidth]{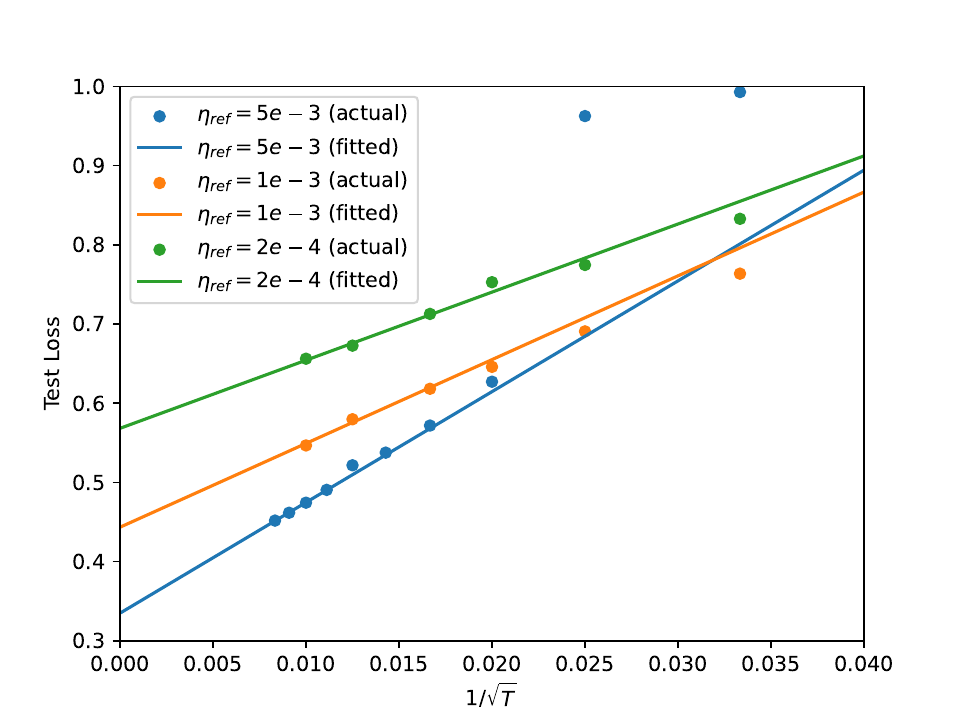}
\vspace{-0.3cm}
    \caption{Loss values (dots) and $1/\sqrt{T}$ prediction for multi-modal VLM with AdamW. Right plot is zoomed.}
    \label{fig:multimodal}
\vspace{-0.3cm}
\end{figure}

\section{Conclusion}
This paper presents a path, starting from a rigorous analysis that is restricted to SGD and convex loss, and generalizing to non-convex deep learning with general optimizers beyond the support of theory. Along the path, we heavily rely on data-driven method to fit our prediction of loss and learning rate. Our findings support (I) convex-like behavior in deep learning, which can be characterized by sequence-to-sequence prediction in \eqref{abs:dl last lr}; (II) asymptotic prediction of $O(1/\sqrt{T})$ loss convergence and $O(1/\sqrt{T})$ learning rate in deep learning; (III) a scaling law that extrapolates across training horizons and model sizes.
We note some limitations of this paper, including the fact that our approach fails to predict test loss but continues to predict training loss when overfitting is severe (see \Cref{fig:vision multi run}), and the lack of understanding why convex-like behaviors exist in various architectures and how many iterations it takes for deep learning to be characterizable by such behaviors.


\bibliography{references}

@article{kaplan2020scaling,
  title={Scaling laws for neural language models},
  author={Kaplan, Jared and McCandlish, Sam and Henighan, Tom and Brown, Tom B and Chess, Benjamin and Child, Rewon and Gray, Scott and Radford, Alec and Wu, Jeffrey and Amodei, Dario},
  journal={arXiv preprint arXiv:2001.08361},
  year={2020}
}

@article{dubey2024llama,
  title={The llama 3 herd of models},
  author={Dubey, Abhimanyu and Jauhri, Abhinav and Pandey, Abhinav and Kadian, Abhishek and Al-Dahle, Ahmad and Letman, Aiesha and Mathur, Akhil and Schelten, Alan and Yang, Amy and Fan, Angela and others},
  journal={arXiv e-prints},
  pages={arXiv--2407},
  year={2024}
}

@misc{wiedmann2025nanovlm,
  author = {Luis Wiedmann and Aritra Roy Gosthipaty and Andrés Marafioti},
  title = {nanoVLM},
  year = {2025},
  publisher = {GitHub},
  journal = {GitHub repository},
  howpublished = {\url{https://github.com/huggingface/nanoVLM}}
}

@inproceedings{
hu2022lora,
title={Lo{RA}: Low-Rank Adaptation of Large Language Models},
author={Edward J Hu and Yelong Shen and Phillip Wallis and Zeyuan Allen-Zhu and Yuanzhi Li and Shean Wang and Lu Wang and Weizhu Chen},
booktitle={International Conference on Learning Representations},
year={2022},
url={https://openreview.net/forum?id=nZeVKeeFYf9}
}

@misc{allal2025smollm2smolgoesbig,
      title={SmolLM2: When Smol Goes Big -- Data-Centric Training of a Small Language Model}, 
      author={Loubna Ben Allal and Anton Lozhkov and Elie Bakouch and Gabriel Martín Blázquez and Guilherme Penedo and Lewis Tunstall and Andrés Marafioti and Hynek Kydlíček and Agustín Piqueres Lajarín and Vaibhav Srivastav and Joshua Lochner and Caleb Fahlgren and Xuan-Son Nguyen and Clémentine Fourrier and Ben Burtenshaw and Hugo Larcher and Haojun Zhao and Cyril Zakka and Mathieu Morlon and Colin Raffel and Leandro von Werra and Thomas Wolf},
      year={2025},
      eprint={2502.02737},
      archivePrefix={arXiv},
      primaryClass={cs.CL},
      url={https://arxiv.org/abs/2502.02737}, 
}

@misc{tschannen2025siglip2multilingualvisionlanguage,
      title={SigLIP 2: Multilingual Vision-Language Encoders with Improved Semantic Understanding, Localization, and Dense Features}, 
      author={Michael Tschannen and Alexey Gritsenko and Xiao Wang and Muhammad Ferjad Naeem and Ibrahim Alabdulmohsin and Nikhil Parthasarathy and Talfan Evans and Lucas Beyer and Ye Xia and Basil Mustafa and Olivier Hénaff and Jeremiah Harmsen and Andreas Steiner and Xiaohua Zhai},
      year={2025},
      eprint={2502.14786},
      archivePrefix={arXiv},
      primaryClass={cs.CV},
      url={https://arxiv.org/abs/2502.14786}, 
}

@misc{laurençon2024matters,
      title={What matters when building vision-language models?}, 
      author={Hugo Laurençon and Léo Tronchon and Matthieu Cord and Victor Sanh},
      year={2024},
      eprint={2405.02246},
      archivePrefix={arXiv},
      primaryClass={cs.CV}
}

@article{hoerl1970ridge,
  title={Ridge regression: Biased estimation for nonorthogonal problems},
  author={Hoerl, Arthur E and Kennard, Robert W},
  journal={Technometrics},
  volume={12},
  number={1},
  pages={55--67},
  year={1970},
  publisher={Taylor \& Francis}
}

@article{krogh1991simple,
  title={A simple weight decay can improve generalization},
  author={Krogh, Anders and Hertz, John},
  journal={Advances in neural information processing systems},
  volume={4},
  year={1991}
}

@inproceedings{sutskever2013importance,
  title={On the importance of initialization and momentum in deep learning},
  author={Sutskever, Ilya and Martens, James and Dahl, George and Hinton, Geoffrey},
  booktitle={International conference on machine learning},
  pages={1139--1147},
  year={2013},
  organization={PMLR}
}

@article{defazio2023optimal,
  title={Optimal linear decay learning rate schedules and further refinements},
  author={Defazio, Aaron and Cutkosky, Ashok and Mehta, Harsh and Mishchenko, Konstantin},
  journal={arXiv preprint arXiv:2310.07831},
  year={2023}
}

@inproceedings{zhousgd,
  title={SGD Converges to Global Minimum in Deep Learning via Star-convex Path},
  author={Zhou, Yi and Yang, Junjie and Zhang, Huishuai and Liang, Yingbin and Tarokh, Vahid},
  booktitle={International Conference on Learning Representations},
  year={2019}
}

@inproceedings{schaippsurprising,
  title={The Surprising Agreement Between Convex Optimization Theory and Learning-Rate Scheduling for Large Model Training},
  author={Schaipp, Fabian and H{\"a}gele, Alexander and Taylor, Adrien and Simsekli, Umut and Bach, Francis},
  booktitle={Forty-second International Conference on Machine Learning},


  year={2025}
}

@misc{Gokaslan2019OpenWeb,
    title={OpenWebText Corpus},
    author={Gokaslan, Aaron and Cohen, Vanya and Pavlick, Ellie and Tellex, Stefanie},
    howpublished={\url{http://Skylion007.github.io/OpenWebTextCorpus}},
    year={2019}
}

@article{radford2019language,
  title={Language models are unsupervised multitask learners},
  author={Radford, Alec and Wu, Jeffrey and Child, Rewon and Luan, David and Amodei, Dario and Sutskever, Ilya and others},
  journal={OpenAI blog},
  volume={1},
  number={8},
  pages={9},
  year={2019}
}

@inproceedings{deng2009imagenet,
  title={Imagenet: A large-scale hierarchical image database},
  author={Deng, Jia and Dong, Wei and Socher, Richard and Li, Li-Jia and Li, Kai and Fei-Fei, Li},
  booktitle={2009 IEEE conference on computer vision and pattern recognition},
  pages={248--255},
  year={2009},
  organization={Ieee}
}

@inproceedings{he2016deep,
  title={Deep residual learning for image recognition},
  author={He, Kaiming and Zhang, Xiangyu and Ren, Shaoqing and Sun, Jian},
  booktitle={Proceedings of the IEEE conference on computer vision and pattern recognition},
  pages={770--778},
  year={2016}
}

@inproceedings{
bu2025gradient,
title={Gradient descent with generalized Newton{\textquoteright}s method},
author={Zhiqi Bu and Shiyun Xu},
booktitle={The Thirteenth International Conference on Learning Representations},
year={2025},
url={https://openreview.net/forum?id=bI3fcTsKW4}
}

@article{jacot2018neural,
  title={Neural tangent kernel: Convergence and generalization in neural networks},
  author={Jacot, Arthur and Gabriel, Franck and Hongler, Cl{\'e}ment},
  journal={Advances in neural information processing systems},
  volume={31},
  year={2018}
}

@article{lee2019wide,
  title={Wide neural networks of any depth evolve as linear models under gradient descent},
  author={Lee, Jaehoon and Xiao, Lechao and Schoenholz, Samuel and Bahri, Yasaman and Novak, Roman and Sohl-Dickstein, Jascha and Pennington, Jeffrey},
  journal={Advances in neural information processing systems},
  volume={32},
  year={2019}
}

@article{bi2024deepseek,
  title={Deepseek llm: Scaling open-source language models with longtermism},
  author={Bi, Xiao and Chen, Deli and Chen, Guanting and Chen, Shanhuang and Dai, Damai and Deng, Chengqi and Ding, Honghui and Dong, Kai and Du, Qiushi and Fu, Zhe and others},
  journal={arXiv preprint arXiv:2401.02954},
  year={2024}
}

@article{li2025predictable,
  title={Predictable Scale: Part I--Optimal Hyperparameter Scaling Law in Large Language Model Pretraining},
  author={Li, Houyi and Zheng, Wenzhen and Hu, Jingcheng and Wang, Qiufeng and Zhang, Hanshan and Wang, Zili and Xuyang, Shijie and Fan, Yuantao and Zhou, Shuigeng and Zhang, Xiangyu and others},
  journal={arXiv preprint arXiv:2503.04715},
  year={2025}
}

@article{zhang2024transformers,
  title={Why transformers need adam: A hessian perspective},
  author={Zhang, Yushun and Chen, Congliang and Ding, Tian and Li, Ziniu and Sun, Ruoyu and Luo, Zhiquan},
  journal={Advances in neural information processing systems},
  volume={37},
  pages={131786--131823},
  year={2024}
}

@inproceedings{sankar2021deeper,
  title={A deeper look at the hessian eigenspectrum of deep neural networks and its applications to regularization},
  author={Sankar, Adepu Ravi and Khasbage, Yash and Vigneswaran, Rahul and Balasubramanian, Vineeth N},
  booktitle={Proceedings of the AAAI Conference on Artificial Intelligence},
  volume={35},
  number={11},
  pages={9481--9488},
  year={2021}
}

@inproceedings{yao2020pyhessian,
  title={PyHessian: Neural Networks Through the Lens of the Hessian},
  author={Yao, Zhewei and Gholami, Amir and Keutzer, Kurt and Mahoney, Michael W},
  booktitle={IEEE BigData},
  year={2020}
}

@article{papyan2018full,
  title={The full spectrum of deepnet hessians at scale: Dynamics with sgd training and sample size},
  author={Papyan, Vardan},
  journal={arXiv preprint arXiv:1811.07062},
  year={2018}
}

@inproceedings{allen2019convergence,
  title={A convergence theory for deep learning via over-parameterization},
  author={Allen-Zhu, Zeyuan and Li, Yuanzhi and Song, Zhao},
  booktitle={International conference on machine learning},
  pages={242--252},
  year={2019},
  organization={PMLR}
}

@inproceedings{zhong2022improving,
  title={Improving Sharpness-Aware Minimization with Fisher Mask for Better Generalization on Language Models},
  author={Zhong, Qihuang and Ding, Liang and Shen, Li and Mi, Peng and Liu, Juhua and Du, Bo and Tao, Dacheng},
  booktitle={Findings of the Association for Computational Linguistics: EMNLP 2022},
  pages={4064--4085},
  year={2022}
}

@article{chen2025understanding,
  title={Understanding Pre-training and Fine-tuning from Loss Landscape Perspectives},
  author={Chen, Huanran and Dong, Yinpeng and Wei, Zeming and Huang, Yao and Zhang, Yichi and Su, Hang and Zhu, Jun},
  journal={arXiv preprint arXiv:2505.17646},
  year={2025}
}

@article{lee2024fp8,
  title={To fp8 and back again: Quantifying the effects of reducing precision on llm training stability},
  author={Lee, Joonhyung and Bae, Jeongin and Kim, Byeongwook and Kwon, Se Jung and Lee, Dongsoo},
  journal={arXiv preprint arXiv:2405.18710},
  year={2024}
}

@article{im2016empirical,
  title={An empirical analysis of the optimization of deep network loss surfaces},
  author={Im, Daniel Jiwoong and Tao, Michael and Branson, Kristin},
  journal={arXiv preprint arXiv:1612.04010},
  year={2016}
}

@article{li2018visualizing,
  title={Visualizing the loss landscape of neural nets},
  author={Li, Hao and Xu, Zheng and Taylor, Gavin and Studer, Christoph and Goldstein, Tom},
  journal={Advances in neural information processing systems},
  volume={31},
  year={2018}
}

@article{besiroglu2024chinchilla,
  title={Chinchilla scaling: A replication attempt},
  author={Besiroglu, Tamay and Erdil, Ege and Barnett, Matthew and You, Josh},
  journal={arXiv preprint arXiv:2404.10102},
  year={2024}
}

@inproceedings{hoffmann2022training,
  title={Training compute-optimal large language models},
  author={Hoffmann, Jordan and Borgeaud, Sebastian and Mensch, Arthur and Buchatskaya, Elena and Cai, Trevor and Rutherford, Eliza and de Las Casas, Diego and Hendricks, Lisa Anne and Welbl, Johannes and Clark, Aidan and others},
  booktitle={Proceedings of the 36th International Conference on Neural Information Processing Systems},
  pages={30016--30030},
  year={2022}
}

@inproceedings{defazio2023learning,
  title={Learning-rate-free learning by d-adaptation},
  author={Defazio, Aaron and Mishchenko, Konstantin},
  booktitle={International Conference on Machine Learning},
  pages={7449--7479},
  year={2023},
  organization={PMLR}
}

@article{mishchenko2023prodigy,
  title={Prodigy: An expeditiously adaptive parameter-free learner},
  author={Mishchenko, Konstantin and Defazio, Aaron},
  journal={arXiv preprint arXiv:2306.06101},
  year={2023}
}

@inproceedings{ivgi2023dog,
  title={DoG is SGD’s best friend: A parameter-free dynamic step size schedule},
  author={Ivgi, Maor and Hinder, Oliver and Carmon, Yair},
  booktitle={International Conference on Machine Learning},
  pages={14465--14499},
  year={2023},
  organization={PMLR}
}

@article{khaled2023dowg,
  title={DoWG unleashed: An efficient universal parameter-free gradient descent method},
  author={Khaled, Ahmed and Mishchenko, Konstantin and Jin, Chi},
  journal={Advances in Neural Information Processing Systems},
  volume={36},
  pages={6748--6769},
  year={2023}
}

@inproceedings{
boreiko2025towards,
title={Towards understanding of orthogonalization in Muon},
author={Valentyn Boreiko and Zhiqi Bu and Sheng Zha},
booktitle={High-dimensional Learning Dynamics 2025},
year={2025},
url={https://openreview.net/forum?id=ppmyFtr9EW}
}

@misc{jordan2024muon,
  author       = {Keller Jordan and Yuchen Jin and Vlado Boza and You Jiacheng and
                  Franz Cesista and Laker Newhouse and Jeremy Bernstein},
  title        = {Muon: An optimizer for hidden layers in neural networks},
  year         = {2024},
  url          = {https://kellerjordan.github.io/posts/muon/}
}

@article{xing2018walk,
  title={A walk with sgd},
  author={Xing, Chen and Arpit, Devansh and Tsirigotis, Christos and Bengio, Yoshua},
  journal={arXiv preprint arXiv:1802.08770},
  year={2018}
}

@article{rae2021scaling,
  title={Scaling language models: Methods, analysis \& insights from training gopher},
  author={Rae, Jack W and Borgeaud, Sebastian and Cai, Trevor and Millican, Katie and Hoffmann, Jordan and Song, Francis and Aslanides, John and Henderson, Sarah and Ring, Roman and Young, Susannah and others},
  journal={arXiv preprint arXiv:2112.11446},
  year={2021}
}

@inproceedings{wang2024scaling,
  title={Scaling Laws Across Model Architectures: A Comparative Analysis of Dense and MoE Models in Large Language Models},
  author={Wang, Siqi and Chen, Zhengyu and Li, Bei and He, Keqing and Zhang, Min and Wang, Jingang},
  booktitle={EMNLP},
  year={2024}
}

@article{porian2024resolving,
  title={Resolving discrepancies in compute-optimal scaling of language models},
  author={Porian, Tomer and Wortsman, Mitchell and Jitsev, Jenia and Schmidt, Ludwig and Carmon, Yair},
  journal={Advances in Neural Information Processing Systems},
  volume={37},
  pages={100535--100570},
  year={2024}
}

@misc{karpathy_nanogpt,
  author       = {Andrej Karpathy},
  title        = {nanoGPT},
  year         = {2023},
  howpublished = {\url{https://github.com/karpathy/nanoGPT}}
}

@article{mei2018mean,
  title={A mean field view of the landscape of two-layer neural networks},
  author={Mei, Song and Montanari, Andrea and Nguyen, Phan-Minh},
  journal={Proceedings of the National Academy of Sciences},
  volume={115},
  number={33},
  pages={E7665--E7671},
  year={2018},
  publisher={National Academy of Sciences}
}

@article{chizat2018global,
  title={On the global convergence of gradient descent for over-parameterized models using optimal transport},
  author={Chizat, Lenaic and Bach, Francis},
  journal={Advances in neural information processing systems},
  volume={31},
  year={2018}
}

@article{du2018gradient,
  title={Gradient descent provably optimizes over-parameterized neural networks},
  author={Du, Simon S and Zhai, Xiyu and Poczos, Barnabas and Singh, Aarti},
  journal={arXiv preprint arXiv:1810.02054},
  year={2018}
}

@article{li2018learning,
  title={Learning overparameterized neural networks via stochastic gradient descent on structured data},
  author={Li, Yuanzhi and Liang, Yingyu},
  journal={Advances in neural information processing systems},
  volume={31},
  year={2018}
}

@inproceedings{du2019gradient,
  title={Gradient descent finds global minima of deep neural networks},
  author={Du, Simon and Lee, Jason and Li, Haochuan and Wang, Liwei and Zhai, Xiyu},
  booktitle={International conference on machine learning},
  pages={1675--1685},
  year={2019},
  organization={PMLR}
}

@article{allen2019learning,
  title={Learning and generalization in overparameterized neural networks, going beyond two layers},
  author={Allen-Zhu, Zeyuan and Li, Yuanzhi and Liang, Yingyu},
  journal={Advances in neural information processing systems},
  volume={32},
  year={2019}
}

@article{fang2019over,
  title={Over parameterized two-level neural networks can learn near optimal feature representations},
  author={Fang, Cong and Dong, Hanze and Zhang, Tong},
  journal={arXiv preprint arXiv:1910.11508},
  year={2019}
}

@inproceedings{dosovitskiy2020image,
  title={An Image is Worth 16x16 Words: Transformers for Image Recognition at Scale},
  author={Dosovitskiy, Alexey and Beyer, Lucas and Kolesnikov, Alexander and Weissenborn, Dirk and Zhai, Xiaohua and Unterthiner, Thomas and Dehghani, Mostafa and Minderer, Matthias and Heigold, G and Gelly, S and others},
  booktitle={International Conference on Learning Representations},
  year={2020}
}

@article{fang2022convex,
  title={Convex formulation of overparameterized deep neural networks},
  author={Fang, Cong and Gu, Yihong and Zhang, Weizhong and Zhang, Tong},
  journal={IEEE Transactions on Information Theory},
  volume={68},
  number={8},
  pages={5340--5352},
  year={2022},
  publisher={IEEE}
}

@inproceedings{loshchilovdecoupled,
  title={Decoupled Weight Decay Regularization},
  author={Loshchilov, Ilya and Hutter, Frank},
  booktitle={International Conference on Learning Representations},
  year={2019}
}

@article{garipov2018loss,
  title={Loss surfaces, mode connectivity, and fast ensembling of dnns},
  author={Garipov, Timur and Izmailov, Pavel and Podoprikhin, Dmitrii and Vetrov, Dmitry P and Wilson, Andrew G},
  journal={Advances in neural information processing systems},
  volume={31},
  year={2018}
}

@inproceedings{choromanska2015loss,
  title={The loss surfaces of multilayer networks},
  author={Choromanska, Anna and Henaff, Mikael and Mathieu, Michael and Arous, G{\'e}rard Ben and LeCun, Yann},
  booktitle={Artificial intelligence and statistics},
  pages={192--204},
  year={2015},
  organization={PMLR}
}

@article{dauphin2014identifying,
  title={Identifying and attacking the saddle point problem in high-dimensional non-convex optimization},
  author={Dauphin, Yann N and Pascanu, Razvan and Gulcehre, Caglar and Cho, Kyunghyun and Ganguli, Surya and Bengio, Yoshua},
  journal={Advances in neural information processing systems},
  volume={27},
  year={2014}
}

@inproceedings{jin2017escape,
  title={How to escape saddle points efficiently},
  author={Jin, Chi and Ge, Rong and Netrapalli, Praneeth and Kakade, Sham M and Jordan, Michael I},
  booktitle={International conference on machine learning},
  pages={1724--1732},
  year={2017},
  organization={PMLR}
}

@article{sagun2016eigenvalues,
  title={Eigenvalues of the hessian in deep learning: Singularity and beyond},
  author={Sagun, Levent and Bottou, Leon and LeCun, Yann},
  journal={arXiv preprint arXiv:1611.07476},
  year={2016}
}

@article{bjorck2024scaling,
  title={Scaling optimal LR across token horizons},
  author={Bjorck, Johan and Benhaim, Alon and Chaudhary, Vishrav and Wei, Furu and Song, Xia},
  journal={arXiv preprint arXiv:2409.19913},
  year={2024}
}

@article{hagele2024scaling,
  title={Scaling laws and compute-optimal training beyond fixed training durations},
  author={H{\"a}gele, Alex and Bakouch, Elie and Kosson, Atli and Von Werra, Leandro and Jaggi, Martin and others},
  journal={Advances in Neural Information Processing Systems},
  volume={37},
  pages={76232--76264},
  year={2024}
}

@article{yang2022tensor,
  title={Tensor programs v: Tuning large neural networks via zero-shot hyperparameter transfer},
  author={Yang, Greg and Hu, Edward J and Babuschkin, Igor and Sidor, Szymon and Liu, Xiaodong and Farhi, David and Ryder, Nick and Pachocki, Jakub and Chen, Weizhu and Gao, Jianfeng},
  journal={arXiv preprint arXiv:2203.03466},
  year={2022}
}

@online{deep_learning_bible_nnfs,
  title        = {“Deep Learning Bible – 1. from Scratch: Z\_10. Optimizers”},
  author       = {{WikiDocs}},
  year         = {2025},
  url          = {https://wikidocs.net/182297},
  urldate      = {2025-11-16},
  note         = {Accessed Nov.\ 16, 2025}
}
\bibliographystyle{assets/plainnat}

\newpage
\appendix

\section{Proofs}
\label{app:proofs}
\subsection{Derivation of \Cref{tab:theorem1}}
\begin{theorem}
\label{col:last-iter}
For SGD under \Cref{def:convex and Lipschitz},  \Cref{eq:last lr array} gives:

\begin{itemize}
    \item[\textbf{Case 1:}] \textbf{Constant learning rate} (\(\eta_t = \etap\) for all \(t\)):
    \[
        \E L(\w_T) \lesssim L_* + \frac{D^2}{2T \etap} + \frac{\etap G^2}{2} \ln T
        := L_{\text{SGD-last-constant}}(\etap; T)
    \]
    The upper bound is minimized at:
    \begin{align}
        \min_{\eta} L_{\text{SGD-last-constant}}(\etap; T)
        = L_* + D G \sqrt{\frac{\ln T}{T}},
        \quad \text{with } \etap^*(T) = \frac{D}{G \sqrt{\ln T\cdot T}}.
    \end{align}

    \item[\textbf{Case 2:}] \textbf{Square-root inverse learning rate} (\(\eta_t = \etap/\sqrt{t}\)):
    \[
        \E L(\w_T) \lesssim L_* + \frac{D^2}{4\sqrt{T}\etap} + \frac{\etap G^2\ln T}{4\sqrt{T}}
        := L_{\text{SGD-last-sqrt-inv}}(\etap; T)
    \]
    The upper bound is minimized at:
    \begin{align}
        \min_{\etap} L_{\text{SGD-last-sqrt-inv}}(\etap; T)
        = L_* + DG\sqrt{\frac{\ln T}{4T}},
        \quad \text{with } \etap^*(T) = \frac{D}{G\sqrt{\ln T}}.
    \end{align}
    
    \item[\textbf{Case 3:}] \textbf{Linearly decaying learning rate} (\(\eta_t = \etap(1 - t/T)\)):
    \[
        \E L(\w_T) \lesssim L_* + \frac{D^2}{T \etap} + \etap G^2
        := L_{\text{SGD-last-linear}}(\etap; T)
    \]
    The upper bound is minimized at:
    \begin{align}
        \min_{\etap} L_{\text{SGD-last-linear}}(\etap; T)
        = L_* + 2DG \sqrt{\frac{1}{T}},
        \quad \text{with } \etap^*(T) = \frac{D}{G \sqrt{T}}.
    \end{align}
    \item[\textbf{Case 4:}] \textbf{Cosine decaying learning rate} (\(\eta_t = \etap\frac{1 + \cos(\pi t/T)}{2}\)):
    \[
        \E L(\w_T) \lesssim L_* + \frac{D^2}{T \etap} + \etap G^2 \cdot 1.061
        := L_{\text{SGD-last-cosine}}(\etap; T)
    \]
    The upper bound is minimized at:
    \begin{align}
        \min_{\etap} L_{\text{SGD-last-cosine}}(\etap; T)
        = L_* + 2DG \sqrt{\frac{1.061}{T}},
        \quad \text{with } \etap^*(T) = \frac{D}{G \sqrt{1.061T}}.
    \end{align}

    \item[\textbf{Case 5:}] \textbf{Warmup-stable-decay learning rate} ( $\eta_t = \begin{cases}
             \etap &\text{if }t<cT\\
              \etap\frac{T-t}{T-cT}&\text{if }t\geq cT
          \end{cases}$):
    \[
        \E L(\w_T) \lesssim L_* + \frac{D^2}{(1+c)\etap T}+ \etap G^2
\left[1+\frac{1}{2}\ln\left(\frac{1+c}{1-c}\right)\right]
        := L_{\text{SGD-last-wsd}}(\etap; T).
    \]
    The upper bound is minimized at:
    \begin{align}
        \min_{\etap} L_{\text{SGD-last-wsd}}(\etap; T)
        =  L_* +  2DG \sqrt{\frac{1+\frac{1}{2}\ln\left(\frac{1+c}{1-c}\right)}{(1+c)T}} ,
        \quad \\ 
        \text{with } \quad \etap^*(T) = \frac{D}{G} \sqrt{(1+c)\left(1+\frac{1}{2}\ln\left(\frac{1+c}{1-c}\right)\right) T}.
    \end{align}
\end{itemize}
\end{theorem}

\begin{proof}[Proof of \Cref{col:last-iter}]
For the ease of presentation, we work in the continuous regime and translate \eqref{eq:lr series} to
$$L_\text{SGD-any}(\tau)= L_*+\frac{D^2}{2\int_0^\tau \eta_t dt}+ \frac{\int_0^\tau \eta_t^2 dt}{2\int_0^\tau \eta_t dt}G^2+\frac{ G^2}{2}\int_0^{\tau-1}\frac{\eta_k\left(\int_{k-1}^\tau \eta_t^2 dt\right)}{\int_k^\tau \eta_t dt \int_{k-1}^\tau \eta_t dt}dk.$$
We will later show \Cref{def:qualifying} is also the continuous version of \eqref{eq:lr series}.


\paragraph{Constant learning rate} $\eta_t=\eta$. Hence, $\int_k^T \eta_t dt=(T-k)\eta$, $\int_k^T \eta_t^2 dt=(T-k)\eta^2$. Then
\begin{align*}
    L_\text{SGD-last-constant}(T)-L_*&=\frac{D^2}{2T\eta}+\frac{G^2\eta}{2}+\frac{ G^2}{2}\int_0^{T-1}\frac{\eta_k\left(\int_{k-1}^T \eta_t^2 dt\right)}{\int_k^T \eta_t dt \int_{k-1}^T \eta_t dt}dk\\
    &=\frac{D^2}{2T\eta}+\frac{G^2\eta}{2}+\frac{ G^2}{2}\int_0^{T-1}\frac{\eta}{T-k}dk\\
    &=\frac{D^2}{2T\eta}+\frac{G^2}{2}\eta\ln T.
\end{align*}

\paragraph{Square-root inverse learning rate} $\eta_t = \etap/\sqrt{t+1}$, then $\int_k^T = 2\eta(\sqrt{T+1}-\sqrt{k+1})$ and $\int_k^T \eta_t^2 = \eta^2 \ln\left( \frac{T+1}{k+1}\right)$, and we can write
\begin{align*}
   L_{\text{SGD-last-sqrt-inv}}(T) - L_* = &\frac{D^2}{4\eta(\sqrt{T+1} - 1)} + \frac{\eta \ln(T+1)}{4(\sqrt{T+1}-1)} G^2 \\
&+\frac{\eta G^2}{8} \int_{0}^{T-1} \frac{ \ln\left(\frac{T+1}{k}\right)}{\sqrt{k+1}(\sqrt{T+1)}-\sqrt{k+1})(\sqrt{T+1}-\sqrt{k})}dk.
\end{align*}
Let $A = \sqrt{T+1}$ and with the change of variable $u = A-\sqrt{k}$, $du = - \frac{1}{2(A-u)} dk$ we can write $\sqrt{T+1}-\sqrt{k} = u$, $\sqrt{T+1}-\sqrt{k+1} = u + O(\frac{1}{A})$, $\sqrt{k+1} = A-u + O(\frac{1}{A})$. Therefore, 
\[
\frac{1}{\sqrt{k+1}(\sqrt{T+1)}-\sqrt{k+1})(\sqrt{T+1}-\sqrt{k})} \approx \frac{1}{(A-u)u^2}
\] and 
\[
\ln\left(\frac{T+1}{k}\right) = -2 \ln\left(1-\frac{u}{A}\right).
\]
Also notice the upper limit $k=T-1$ corresponds to $u \approx \frac{1}{2A}$ and the lower limit is $u=A$, so the integral can be approximated by
\[
I(T) = \int_{0}^{T-1} \frac{ \ln\left(\frac{T+1}{k}\right)}{\sqrt{k+1}(\sqrt{T+1)}-\sqrt{k+1})(\sqrt{T+1}-\sqrt{k})}dk \approx -4 \int_{1/(2A)}^A \frac{\ln(1-u/A)}{u^2} du.
\]
This integral can be calculated exactly with a change of variable $\epsilon = \frac{1}{2A^2}$ as 
\[
-\frac{4}{A} \left( \frac{\ln(1-\epsilon)}{\epsilon} + \ln\epsilon - \ln(1-\epsilon)\right) = -\frac{4}{A}(\ln\epsilon - 1 + o(1)) \approx \frac{4\ln(T+1)+6.7726}{\sqrt{T+1}} + o(1/\sqrt{T}) .
\]
Combining the above we have
\begin{align*}
& L_{\text{SGD-last-sqrt-inv}}(T) - L_* \\
&\approx \frac{D^2}{4\eta(\sqrt{T+1} - 1)} + \frac{\eta \ln(T+1)}{4(\sqrt{T+1}-1)} G^2 + \frac{4\ln(T+1)+6.7726}{\sqrt{T+1}} + o(1/\sqrt{T})
\\
&\sim \frac{D^2}{4\eta\sqrt{T}} + \frac{\eta \ln(T)}{4\sqrt{T}} G^2 + O(\frac{\ln(T)}{\sqrt{T}}).
\end{align*}
\paragraph{Linear decay learning rate} $\eta_t=\eta(1-t/T)$ and it is not hard to get $\int_k^T \eta_t dt=\eta(T-k)^2/2T$, $\int_k^T \eta_t^2 dt=\eta^2(T-k)^3/3T^2$.

We can write down
\begin{align*}
L_\text{SGD-last-linear}(T)-L_*&=\frac{D^2}{\eta T}+ \frac{\eta^2 T/3}{\eta T}G^2+\frac{ G^2}{2}\int_0^{T-1}\frac{\eta_k\eta^2(T-k+1)^3/3T^2}{\eta(T-k)^2/2T *(T-k+1)^2/2T}dk
\\
&=\frac{D^2}{\eta T}+ \frac{\eta G^2}{3}+\frac{2\eta G^2}{3T}\int_0^{T-1}(1+\frac{1}{(T-k)})dk
\\
&=\frac{D^2}{\eta T}+ \frac{\eta G^2}{3}+\frac{2\eta G^2}{3T}(T-1+\ln T)\\
&=\frac{D^2}{\eta T}+ \frac{\eta G^2}{3}(3+2\ln T/T-2/T).
\end{align*}

\paragraph{Cosine decay learning rate} \label{proof:cosine-last-iter}
$\eta_t=\frac{1}{2}(1+\cos(\pi t/T))\eta$.

We derive
\begin{align*}
  A(k):= \int_k^T \eta_t dt &=\eta\int_k^T \frac{1 + \cos\!\bigl(\tfrac{\pi t}{T}\bigr)}{2} \, dt \\
  &= \eta(\frac{1}{2}\Bigl[(T-k)\Bigr]
     \;+\;\frac{1}{2}\int_k^T \cos\!\Bigl(\tfrac{\pi t}{T}\Bigr)\,dt) \\
  &= \eta(\frac{T - k}{2}
     \;+\;\frac{1}{2}\cdot
       \frac{T}{\pi}\Bigl[\sin\!\bigl(\tfrac{\pi t}{T}\bigr)\Bigr]_{t=k}^{t=T})  \\
  &= \eta(\frac{T - k}{2}
     \;-\;\frac{T}{2\pi}\,\sin\!\Bigl(\frac{\pi k}{T}\Bigr)).
\end{align*}

and
\begin{align*}
  B(k):=\int_k^T \eta_t^2 dt
  &= \eta^2\int_k^T
     \Bigl(\tfrac{3}{8}
          + \tfrac{1}{2}\cos\!\tfrac{\pi t}{T}
          + \tfrac{1}{8}\cos\!\tfrac{2\pi t}{T}\Bigr)
     \,dt \\
  &= \eta^2(\frac{3(T-k)}{8}
     + \frac{1}{2}\cdot \frac{T}{\pi}\Bigl[\sin\!\tfrac{\pi t}{T}\Bigr]_{k}^{T}
     + \frac{1}{8}\cdot \frac{T}{2\pi}\Bigl[\sin\!\tfrac{2\pi t}{T}\Bigr]_{k}^{T}) \\
  &= \eta^2(\frac{3(T-k)}{8}
     - \frac{T}{2\pi}\sin\!\Bigl(\tfrac{\pi k}{T}\Bigr)
     - \frac{T}{16\pi}\sin\!\Bigl(\tfrac{2\pi k}{T}\Bigr)).
\end{align*}
Finally, substituting them into main expression gives

\begin{align*}
  &\frac{D^2}{2 A(0)}
  +
  \frac{B(0)}{2 A(0)}\,G^2
 +
  \frac{G^2}{2}
  \int_{0}^{T-1}
  \frac{\eta_k B(k-1)}{A(k)A(k-1)}dk,
  \\
 & =
 \frac{D^2}{T\eta}+\frac{3\eta G^2}{8}+
  \frac{\eta G^2}{2}
  \int_{0}^{T-1}
  \frac{\tfrac12\bigl(1+\cos(\tfrac{\pi k}{T})\bigr)\,B(k-1)}
       {A(k)\,A(k-1)}dk\\
&= \frac{D^2}{T\eta}+\frac{3\eta G^2}{8}+
  \frac{\eta G^2}{4}\int_{0}^{T-1}
\frac{\bigl(1+\cos(\tfrac{\pi k}{T})\bigr)\,
\bigl[\tfrac{3(T-k+1)}{8}
-\tfrac{T}{2\pi}\sin\!\bigl(\tfrac{\pi(k-1)}{T}\bigr)
-\tfrac{T}{16\pi}\sin\!\bigl(\tfrac{2\pi(k-1)}{T}\bigr)\bigr]}
{\bigl(\tfrac{T-k}{2}-\tfrac{T}{2\pi}\sin(\tfrac{\pi k}{T})\bigr)\,
\bigl(\tfrac{T-k+1}{2}-\tfrac{T}{2\pi}\sin\!\bigl(\tfrac{\pi(k-1)}{T}\bigr)\bigr)}
\,dk.
\end{align*}

By a change of variable $x=\frac{k}{T}$, the last integral behaves like 


\[
\begin{aligned}
&T\int_{0}^{1-1/T}\!\Bigl(\frac{(1+\cos(\pi x))B(Tx)}{A^2(Tx)}+O\!\bigl(\tfrac1{T^2}\bigr)\Bigr)\,dx
\\
&=\int_{0}^{1}
\frac{(1+\cos\pi x)\Bigl[\frac{3(1-x)}{8}
-\frac{1}{2\pi}\sin(\pi x)
-\frac{1}{16\pi}\sin(2\pi x)\Bigr]}
     {\Bigl(\frac{1-x}{2}-\frac{1}{2\pi}\sin(\pi x)\Bigr)^2}
\,dx
\;+\;O\!\Bigl(\tfrac1T\Bigr)\\
&\approx 2.7443+\;O(\tfrac1T).\\
\end{aligned}
\]
Therefore, 
$$L_\text{SGD-last-cosine}(T)-L_*\approx \frac{D^2}{T\eta}+\frac{3\eta G^2}{8}+
  \frac{\eta G^2}{4}\times (2.7443 +O(\frac{1}{T})) \approx  \frac{D^2}{T\eta}+\eta G^2 (1.061 +O(\frac{1}{T})).$$

\paragraph{Warmup-stable-decay learning rate}  $\eta_t = \begin{cases}
             \eta &\text{if }t<cT\\
              \eta\frac{T-t}{T-cT}&\text{if }t\geq cT
\end{cases}$. We can get
\[
\int_k^T \eta_t dt = \begin{cases}
    \eta (cT-k) + \frac{(1-c)\eta T}{2}&\text{if } k < cT, \\ \frac{\eta(T-k)^2}{2(T-cT)} &\text{if } k \geq cT,
\end{cases}
\] and 
\[
\int_k^T \eta_t^2 dt = \begin{cases}
    \eta^2 (cT-k) + \frac{\eta^2}{3} (T-cT) &\text{if } k < cT, \\
    \frac{\eta^2}{3(T-cT)^2} (T-k)^3 &\text{if } k \geq cT.
\end{cases}
\]
We can write
\begin{align*}
&L_\text{SGD-last-WSD}(T)-L_*
\\
&= \frac{D^2}{(1+c)\eta T}+ \frac{\frac{1+2c}{3}\eta^2 T}{(1+c)\eta T}G^2
\\
&+\frac{G^2}{2}\left[\int_0^{cT}\frac{\eta_k\left(\int_{k-1}^T \eta_t^2 dt\right)}{\int_k^T \eta_t dt \int_{k-1}^T \eta_t dt}dk+\int_{cT}^{T-1}\frac{\eta_k\left(\int_{k-1}^T \eta_t^2 dt\right)}{\int_k^T \eta_t dt \int_{k-1}^T \eta_t dt}dk\right]
\\
&\approx \frac{D^2}{(1+c)\eta T}+ \frac{\frac{1+2c}{3}\eta^2 T}{(1+c)\eta T}G^2
\\
&+\frac{G^2}{2}\left[\int_0^{cT}\frac{\eta\left(\eta^2 (cT-k) + \frac{\eta^2}{3} (T-cT) \right)}{(\eta (cT-k) + \frac{(1-c)\eta T}{2})^2}dk+\int_{cT}^{T-1}\frac{\eta\frac{T-t}{T-cT}\left(\frac{\eta^2}{3(T-cT)^2} (T-k)^3\right)}{(\frac{\eta(T-k)^2}{2(T-cT)})^2}dk\right]
\\
&= \frac{D^2}{(1+c)\eta T}+ \frac{(1+2c)\eta}{3(1+c)}G^2
+\frac{\eta G^2}{2}\left[\ln\left(\frac{1+c}{1-c}\right)-\frac{2c}{3(1+c)}+\frac{4}{3}\left(1-\frac{1}{(1-c)T}\right)\right]
\\
&\sim \frac{D^2}{(1+c)\eta T}+ \frac{(1+2c)\eta}{3(1+c)}G^2
+\frac{\eta G^2}{2}\left[\ln\left(\frac{1+c}{1-c}\right)-\frac{2c}{3(1+c)}+\frac{4}{3}\right]
\\
&= \frac{D^2}{(1+c)\eta T}+ \eta G^2
\left[1+\frac{1}{2}\ln\left(\frac{1+c}{1-c}\right)\right]
\end{align*}
\end{proof}

\subsection{Derivation of \Cref{def:qualifying}}
\label{app:condition24}
Recall the loss bound at last iterate in \eqref{eq:last lr array} is 
$$L_*+\frac{D^2}{2\sum_{t=1}^T  \eta_t}+ \frac{ G^2}{2}\left(\frac{\sum_{t=1}^T \eta_t^2}{\sum_{t=1}^T  \eta_t}+\sum_{k=1}^{T -1}\frac{\eta_k}{\sum_{t=k+1}^T \eta_t}\frac{\sum_{t=k}^T  \eta_t^2}{\sum_{t=k}^T  \eta_t}\right)
$$
We can rewrite the last term by applying \Cref{thm1}, with $x_t=\eta_t^2$.
\begin{theorem}
\label{thm1}
For any sequence $\{x_t\}$, if $x_T/\eta_T=0$, then
$$\frac{1}{\sum_{t=1}^T  \eta_t}\sum_{t=1}^T  x_t+\sum_{k=1}^{T -1}\frac{\eta_{k}}{\sum_{t=k+1}^T \eta_t \sum_{t=k}^T  \eta_t}\sum_{t=k}^T  x_t=\sum_{t=1}^{T-1}  \left(\frac{x_t}{ \sum_{k=t+1}^T \eta_k}\right)
$$
\end{theorem}
\begin{proof}[Proof of \Cref{thm1}]
We firstly simplify the notation by denoting $A_k=\frac{\eta_{k}}{\sum_{t=k+1}^T \eta_t \sum_{t=k}^T  \eta_t}=\frac{1}{\sum_{t=k+1}^T \eta_t}-\frac{1}{ \sum_{t=k}^T  \eta_t}$ for $k<T$.

With this notation, we write the left hand side of \Cref{thm1} as
$$\frac{1}{\sum_{t=1}^T  \eta_t}\sum_{t=1}^T  x_t+\sum_{k=1}^{T -1} \left(A_k\sum_{t=k}^T  x_t\right)=\frac{1}{\sum_{t=1}^T  \eta_t}\sum_{t=1}^T  x_t+\sum_{k=1}^{T-1} \left(A_k\sum_{t=k}^{T-1}  x_t\right)+\sum_{k=1}^{T-1} A_k x_T$$
Now we can exchange the double sum by noticing that
$$\sum_{k=1}^{T-1} \left(A_k\sum_{t=k}^{T-1}  x_t\right)=\sum_{t=1}^{T-1} \left( x_t\sum_{k=1}^t A_k\right)=\sum_{1\leq k\leq t\leq T-1}A_k x_t$$
Therefore the left hand side of \Cref{thm1} becomes
$$
\frac{1}{\sum_{t=1}^T  \eta_t}\sum_{t=1}^T  x_t+\sum_{t=1}^{T-1} \left( x_t\sum_{k=1}^t A_k\right)+x_T\left(\sum_{k=1}^{T-1} A_k\right)
$$
Splitting the first term, we get
\begin{align}
\sum_{t=1}^{T-1}  x_t\left(\frac{1}{\sum_{t=1}^T  \eta_t}+\sum_{k=1}^t A_k\right)+x_T\left(\frac{1}{\sum_{t=1}^T\eta_t}+\sum_{k=1}^{T-1} A_k\right)
\label{eq:2terms A}
\end{align}
in which the term in big bracket can be greatly simplified by telescoping the sum,
$$\frac{1}{\sum_{t=1}^T  \eta_t}+\sum_{k=1}^t A_k=\frac{1}{\sum_{s=1}^T  \eta_s}+\sum_{k=1}^t \left(\frac{1}{\sum_{s=k+1}^T \eta_s}-\frac{1}{ \sum_{s=k}^T  \eta_s}\right)=\frac{1}{\sum_{s=t+1}^T \eta_s}.$$
In particular, with $t=T-1$, this is just $1/\eta_T$.

Substituting back to right hand side of \eqref{eq:2terms A}, we finally obtain an equivalent form of the left hand side of \Cref{thm1} as follows, after using $x_T/\eta_T=0$: 
$$\sum_{t=1}^{T-1} x_t\left(\frac{1}{\sum_{s=t+1}^T \eta_s}\right)+\frac{x_T}{\eta_T}$$
\end{proof}

Now we have
\begin{align}
L_*+\frac{D^2}{2\sum_{t=1}^T  \eta_t}+
\frac{G^2}{2}\sum_{t=1}^{T-1}  \left(\frac{\eta_t^2}{ \sum_{k=t+1}^T \eta_k}\right)
\label{eq:loss mapping simplified}
\end{align}

However, \eqref{eq:loss mapping simplified} has some singularity in the last term at $t=T$, where the denominator may be zero (if learning rate decays) and not divisible. To work around, we replace the discrete summation by continuous integral to get 
$$L_*+\frac{D^2}{2\int_0^T  \eta_t dt}+
\frac{G^2}{2}\int_0^T  \left(\frac{\eta_t^2}{ \int_{t}^T \eta_k dk}\right)dt$$

\subsection{Summary of generalizations}

\begin{table}[H]
\centering
\caption{Summary of generalizations in this work and coefficients to be fitted in a data-driven way.}
\vspace{-0.3cm}
    \begin{tabular}{c|c|c}
         & related experiments & coefficients to fit \\\hline
 Generalization 1 & Figure 2-5 & $\tilde{L}_\infty, \tilde{D},\tilde{G}$ \\
 Generalization 2 & Figure 6 &$\tilde{L}_\infty, \tilde{q_1}, \tilde{q_2}$ \\
 Generalization 3 & Figure 7, Table 2 & $\tilde{L}_\infty, \tilde{q_1}\tilde{q_2}$ \\
 Generalization 4 & Figure 8-11 & $\tilde{L}_\infty, \tilde{Q}(\etar)$
          \end{tabular}
\label{tab:generalization}
\end{table}

\section{More experiment settings
}\label{app:experiment}
\paragraph{\Cref{fig:sgd multi}:} WSD uses 10\% warmup and 10\% decay.

\paragraph{\Cref{fig:resnet-imagenet}:} 
Optimizer is SGD without momentum or weight decay, batch size $256$.

Linear decays from peak learning rate $0.225$ to 0.
Cosine uses peak learning rate $0.225$.
Cyclic uses 5 triangular cycles and peak learning rate is $0.056$.
WSD uses 10\% warmup and 10\% decay, peak learning rate $0.113$.

\paragraph{\Cref{fig:resnet-imagenet-adamw}}
Optimizer is AdamW: $\beta_1=0.9, \beta_2=0.999$, weight decay$=0.01$, batch size $256$.

Linear decays from peak learning rate $0.0003$ to 0.
Cosine uses peak learning rate $0.0003$.
Cyclic uses 5 triangular cycles and peak learning rate is $0.0003$.
WSD uses 10\% warmup and 10\% decay, peak learning rate $0.00014$.

\paragraph{\Cref{fig:nanogpt-adamw} and \Cref{fig:nanogpt-muon}}

We use two optimizers: AdamW and Muon-NSGD. Here Muon-NSGD uses Muon and NSGD with the same learning rate as in \citep{boreiko2025towards}: denoting $\mathbf{W}$ as a layer's parameter and ignoring weight decay for brevity,
\begin{align*}
\textup{Muon:} & \mathbf{W}_{t+1}=\mathbf{W}_t-\eta\cdot\textup{NS}(\mathbf{m}_t)
\\
\textup{NSGD:} & \mathbf{W}_{t+1}=\mathbf{W}_t-\eta\cdot \mathbf{m}_t/\|\mathbf{m}_t\|_2
\end{align*}
where NS is the Newton-Schulz matrix iteration and $\mathbf{m}_t$ is the momentum.

Each optimizer is trained under four learning rate schedulers, i.e., linear decay with warm-up, cosine decay with warm-up, cyclic and WSD. For schedulers with warm-up, we allocate  2\% of the total steps (i.e., 100 of 5000 steps) to warmup. For WSD, we start to decay the learning rate from the last 10\% of the steps. For AdamW, the peak learning rate is 2e-3. For Muon-NSGD, the peak learning rate is 2e-2.

\paragraph{\Cref{sec:experiments}}
All runs are trained with bf16 mixed-precision training. AdamW uses $\beta_1=0.9,\beta_2=0.95$. Muon-NSGD uses momentum 0.95.
0.1B models use batch size 512. 1B/7B models use batch size 64. We train with 1024 sequence length. 

VLM models use SmolLM2-360M-Instruct \citep{allal2025smollm2smolgoesbig} as the language model and siglip2-base-patch16-512 \citep{tschannen2025siglip2multilingualvisionlanguage} as the vision model. These models use different learning rates instead of a single one. For example, when $\etar=5e-3$, the vision and language backbones use $\etar/\sqrt{T}$ as the learning rate while the modality projector uses $100\etar/\sqrt{T}$. As a consequence, all learning rates are $1/\sqrt{T}$ scaled. AdamW uses default hyperparameters with batch size being 64. Sequence length 8192, max image size 2048, max images per example 8, max images per knapsack  36.

\clearpage
\section{Details of \Cref{sec:4}}
\subsection{\Cref{fig:hitch-hiker}}
\label{app:hitch-hiker}

The model configurations of \Cref{fig:hitch-hiker} can be found in \Cref{table:HH-dense} and \Cref{table:HH-moe}.   
All other parameters not specified here follow the default setting in \citep{li2025predictable}. 

\begin{table}[h]
\centering
\caption{Dense model configurations. Model size is measured in billions of parameters (B). In columns, $d_{\text{model}}$ = hidden dimension, $d_{\text{ff}}$ = feed-forward width, 
$n_{\text{heads}}$ = number of attention heads, $n_{\text{layers}}$ = number of layers, 
$B$ = batch size, $T$ = total iterations.}
\begin{tabular}{ccccccc}
\toprule
Model Size (B) & $d_{\text{model}}$ & $d_{\text{ff}}$ & $n_{\text{heads}}$ & $n_{\text{layers}}$ & $B$ & $T$ \\
\midrule
1.07 & 2048 & 8192 & 16 & 16 & 352 & 27743 \\
0.54 & 1280 & 9048 & 10 & 13 & 352 & 39464 \\
0.43 & 1280 & 9472 & 10 & 10 & 512 & 21696 \\
0.21 &  960 & 9368 & 15 &  7 & 128 & 76293 \\
\bottomrule
\end{tabular}
\label{table:HH-dense}
\end{table}

\begin{table}[h]
\centering
\caption{MoE model configurations. Model size is measured in billions of parameters (B). In columns, $n_{\text{exp}}$ = number of experts, $d_{\text{exp}}$ = per-expert hidden size, 
SED = selected expert dimension, $B$ = batch size, $T$ = total iterations.}
\begin{tabular}{ccccccc}
\toprule
Model Size (B) & $n_{\text{exp}}$ & $d_{\text{exp}}$ & SED & $B$ & $T$ \\
\midrule
2.16 &  8 & 2888 & 8664 & 32 & 30517  \\
2.15 (Green) & 88 &  352 &  704 & 32 & 30517  \\
2.15 (Orange) & 89 &  352 &  352 & 32 & 30517 \\
\bottomrule
\end{tabular}
\label{table:HH-moe}
\end{table}


\subsection{\Cref{fig:vis hitchhiker} and \Cref{tab:replica}}
We give the comprehensive table of all datapoints from \citep{besiroglu2024chinchilla}, where the linear regression fits well in all cases.
\begin{table}[!htb]
\centering
\renewcommand{\arraystretch}{1.2} 
\begin{tabular}{c|c|c|c|c}
  model size(B) &num of horizons  &$2\tilde q_1\tilde q_2$&$\tilde L_\infty$&$R^2$ score\\\hline
0.074&5&3.22e+04&2.825&0.991\\
0.090&3&3.19e+04&2.774&0.991\\
0.106&4&3.38e+04&2.706&1.000\\
0.117&3&3.27e+04&2.692&0.996\\
0.140&7&3.04e+04&2.670&0.991\\
0.163&3&3.11e+04&2.619&1.000\\
0.175&7&3.08e+04&2.619&0.995\\
0.196&4&3.14e+04&2.582&0.999\\
0.217&6&3.54e+04&2.526&0.998\\
0.251&3&3.37e+04&2.517&1.000\\
0.279&8&3.29e+04&2.498&0.999\\
0.305&7&3.14e+04&2.488&0.997\\
0.425&8&3.27e+04&2.430&0.998\\
0.489&4&3.30e+04&2.404&0.999\\
0.552&8&3.24e+04&2.382&0.999\\
0.586&8&3.25e+04&2.368&0.994\\
0.632&8&3.17e+04&2.367&0.998\\
0.664&3&3.46e+04&2.330&0.999\\
0.724&3&3.53e+04&2.320&0.999\\
0.817&10&3.28e+04&2.315&0.994\\
0.893&3&3.35e+04&2.304&0.998\\
1.019&7&3.06e+04&2.305&0.997\\
1.143&10&3.10e+04&2.275&0.998\\
1.265&10&3.05e+04&2.286&0.986\\
1.426&3&4.07e+04&2.214&0.984\\
1.429&9&3.18e+04&2.253&0.996\\
1.592&4&4.22e+04&2.182&0.997\\
1.611&9&3.36e+04&2.228&0.995\\
1.730&7&3.53e+04&2.207&0.998\\
1.795&11&3.41e+04&2.211&0.997\\
2.004&8&3.62e+04&2.178&0.999\\
2.280&7&4.41e+04&2.128&1.000\\
2.636&6&4.08e+04&2.113&0.998\\
2.979&10&5.90e+04&2.016&0.990\\
4.519&6&3.83e+04&2.106&0.978\\
6.792&8&4.66e+04&2.023&0.999\\
9.290&4&4.29e+04&2.046&0.988\\
12.560&3&4.23e+04&2.053&1.000\\
\end{tabular}
\caption{Comprehensive version of \Cref{tab:replica}.}
\label{tab:replica all}
\end{table}

\clearpage

\section{Additional experiments}
\label{app:overfitting}
\subsection{Effect of overfitting}
Our experiments are primarily conducted when the train loss and test loss are similar, without much overfitting. For example, our GPT2 runs in \Cref{sec:experiments} use up to 260B tokens (about 30 epochs of OpenWebText) and test loss is close to train loss.

In order to validate our generalization in the overfitting regime, we train ViT models and ResNet models on ImageNet for multiple runs up to 1 epoch, and a ResNet50 on CIFAR10 for multiple runs up to 820 epochs. As expected, within 1 epoch when overfitting is insignificant, $O(1/\sqrt{T})$ convergence holds for both train loss and test loss, although the coefficients of fitting may be slightly different. However, with overfitting (e.g. after 50 epochs), we observe that $O(1/\sqrt{T})$ convergence only holds for train loss but no longer for test loss (see the last plot in \Cref{fig:vision multi run}).
\begin{figure}[!htb]
    \centering
    \includegraphics[width=0.43\linewidth]{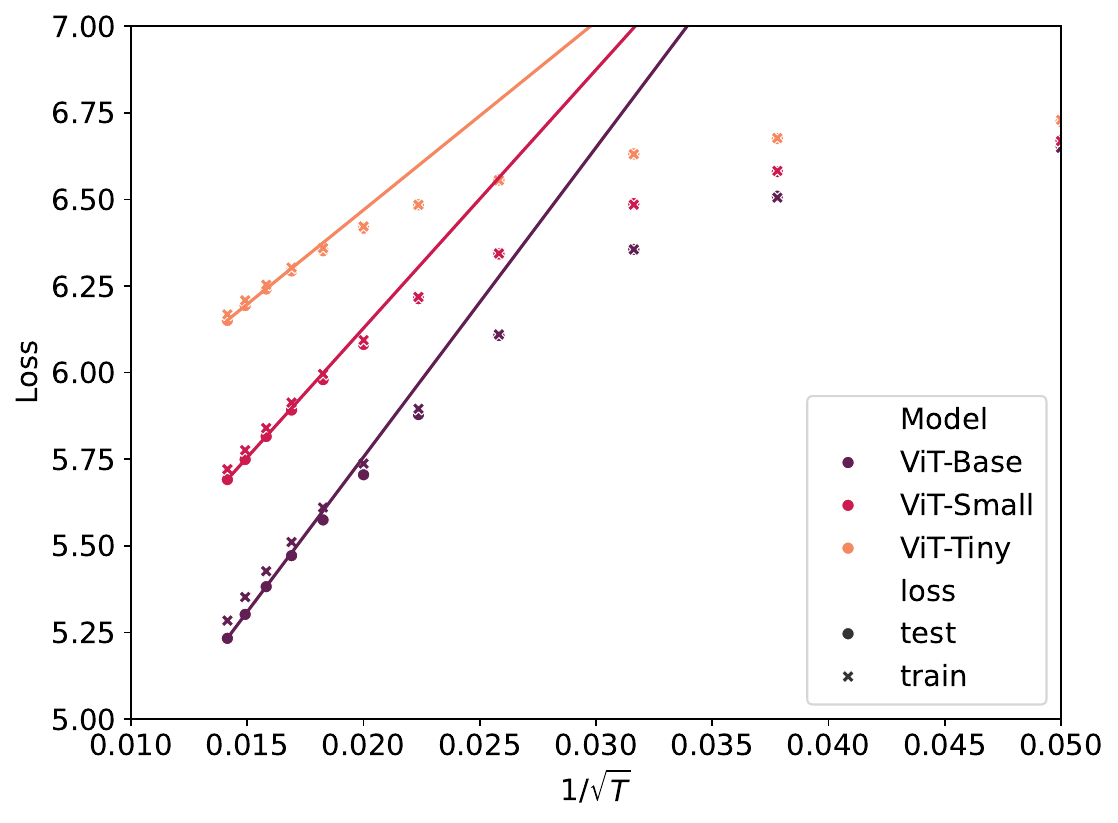}
    \includegraphics[width=0.43\linewidth]{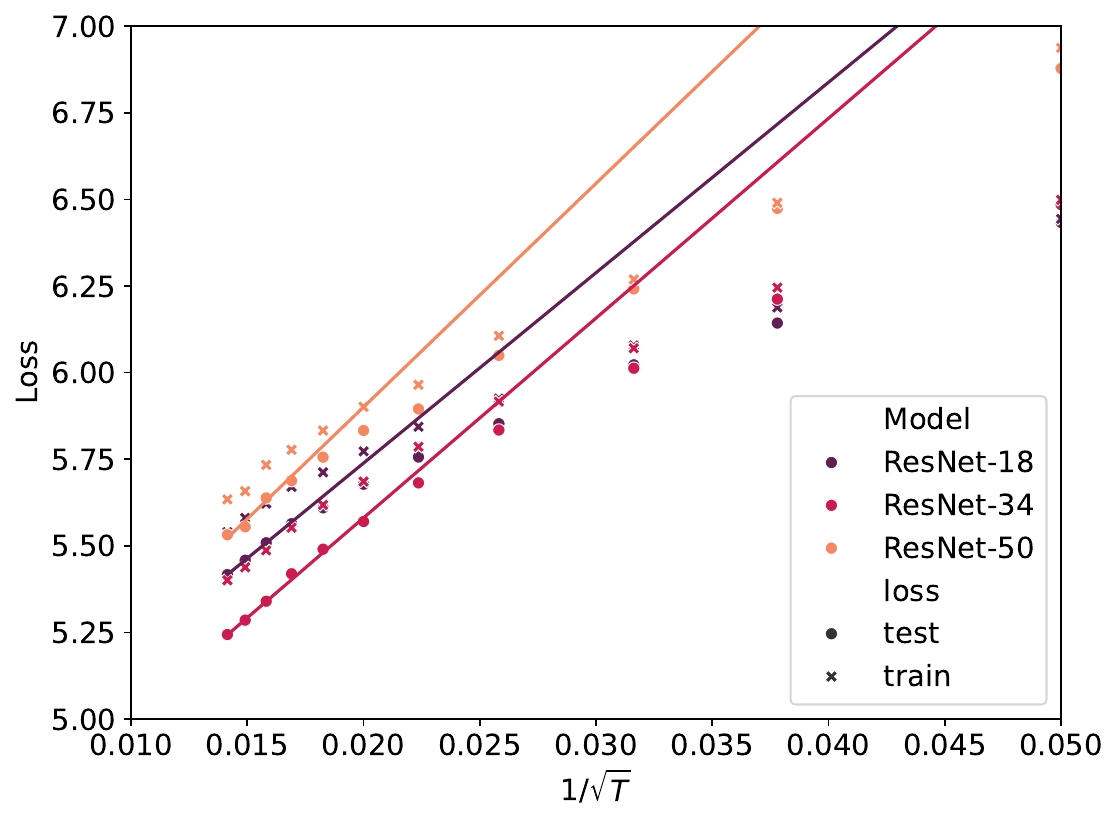} 
    \\
    \includegraphics[width=0.8\linewidth]{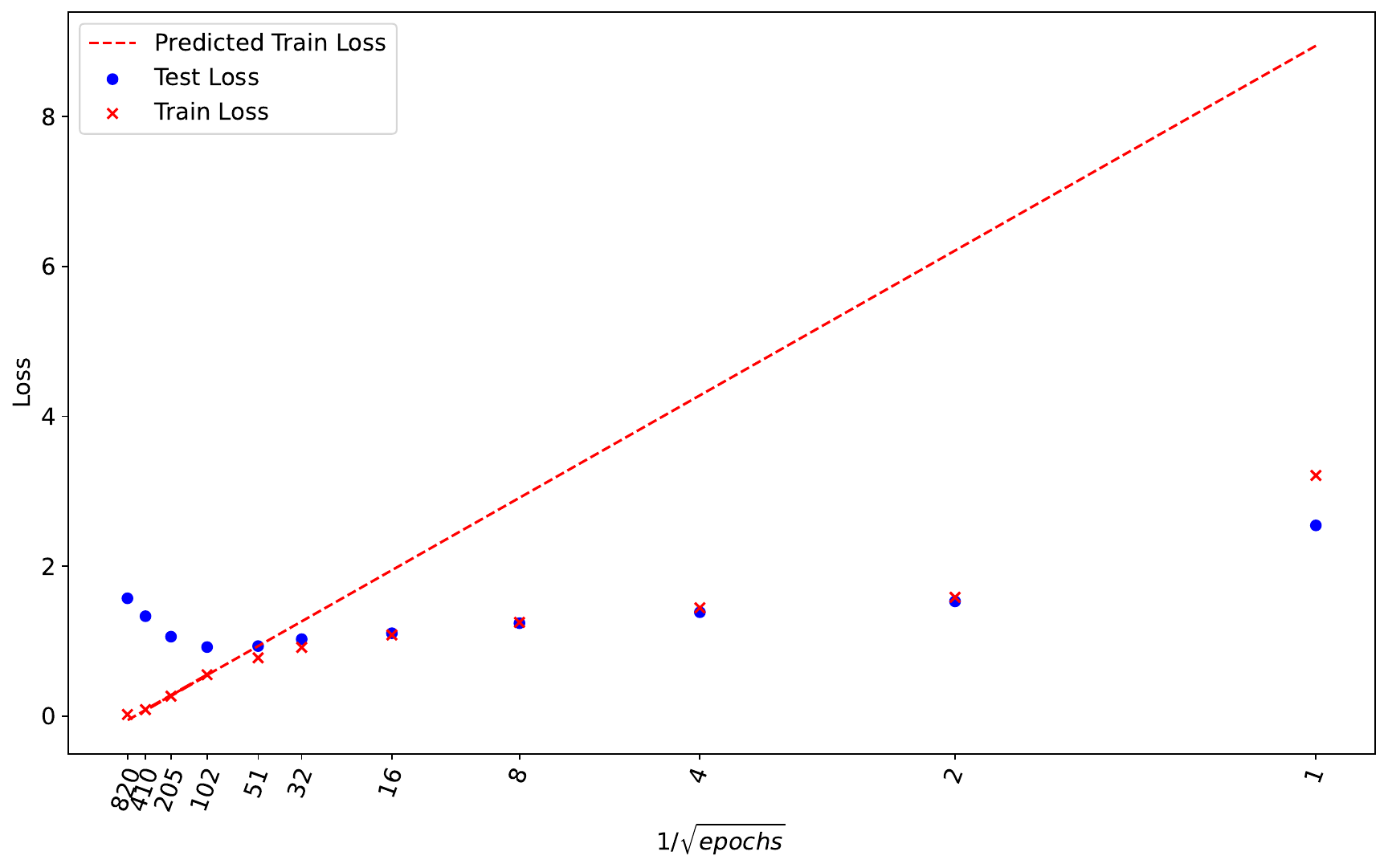}
    \caption{Loss of multiple runs with scaled peak learning rate and AdamW. Upper left: ViT model on ImageNet (within 1 epoch). Upper right: ResNet models on ImageNet (within 1 epoch). Lower: ResNet50 on CIFAR10 (over 1 epoch), where we have marked the number of epochs along x-axis.}
    \label{fig:vision multi run}
\end{figure}

All models are optimized by AdamW with $\beta_1=0.9, \beta_2=0.999$, weight decay$=0.01$, batch size $128$. Learning rate schedule is cosine decay with peak learning rate $0.01/\sqrt{T}$.

\newpage
\subsection{Parameter-efficient v.s. full model training}
We experiment with full model training and parameter-efficient fine-tuning of GPT2 models. We test small/medium/large sizes on E2E dataset. For parameter-efficient fine-tuning, we apply low-rank adaptation (LoRA) with rank 4. Our experiments in \Cref{fig:gpt multi run} show that $O(1/\sqrt{T})$ convergence holds within both training regimes.
All models are optimized by AdamW with  $\beta_1=0.9, \beta_2=0.999$, weight decay$=0.01$, batch size $8$. Learning rate schedule is linear with peak learning rate $0.0005/\sqrt{T}$ for LoRA and $0.00005/\sqrt{T}$ for full model fine-tuning.
\begin{figure}[!htb]
    \centering
    \includegraphics[width=0.41\linewidth]{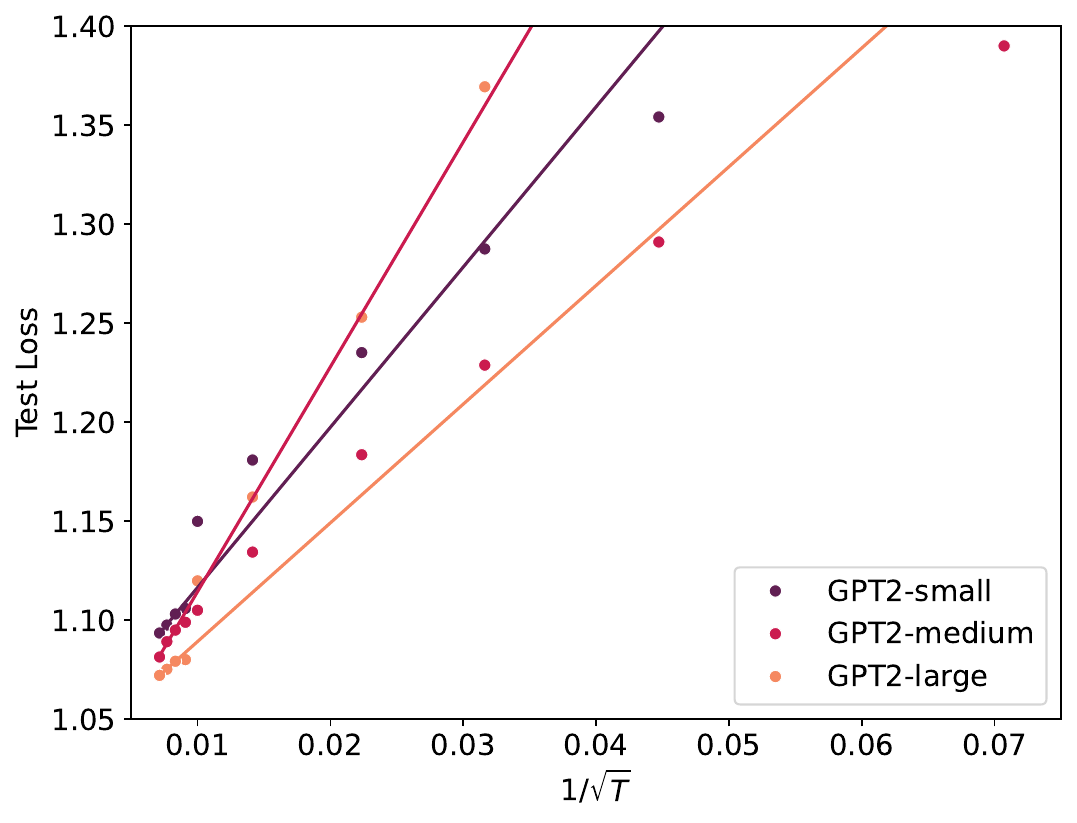}
    \includegraphics[width=0.4\linewidth]{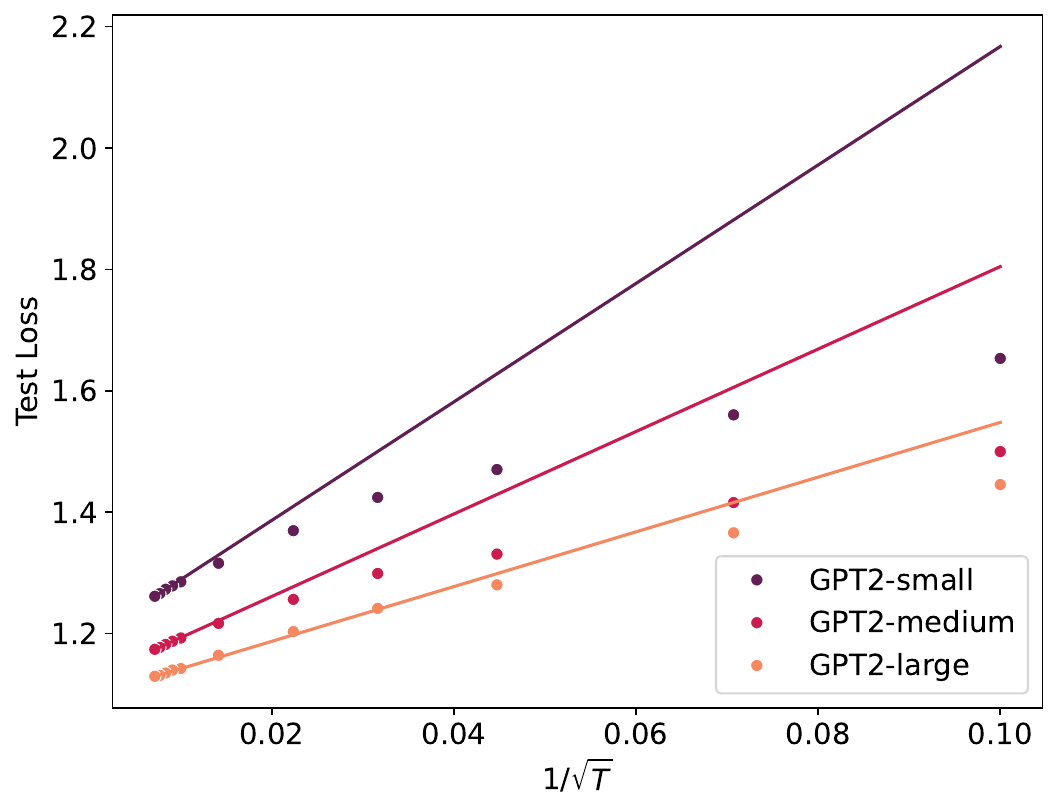} 
    \vspace{-0.3cm}
    \caption{Loss of multiple fine-tuning runs with scaled peak learning rate and AdamW, on E2E dataset. Left: full model fine-tuning. Right: LoRA fine-tuning.}
    \label{fig:gpt multi run}
\end{figure}

\vspace{-0.3cm}
\subsection{Ablation study on hyperparameters}

We further evaluate our \Cref{abstract:4} through the ablation study over five key hyperparameters, such as weight decay, gradient clipping, momentum coefficient, batch size, and random seeds. For each new row in \Cref{tab:ablation}, we launch 5 runs with $T\in [60^2,70^2,80^2,90^2,100^2]$, so in total we summarize the statistics of 35 runs. We fit linear regression of validation loss against $1/\sqrt{T}$, and we consistently observe goodness of fit.

\begin{table}[!htb]
    \centering
    \begin{tabular}{c|c|c|c}
         & $\tilde Q$ estimation (std err)& $\tilde L_\infty$ estimation (std err)&$R^2$ fit score \\\hline
       default (seed=1337)  & 23.9094 (0.469) & 2.8298 (0.004)&0.998 \\
       seed=3333  &24.1820 (0.220)&2.8321 (0.003)& 1.000\\
       seed=8888  & 25.0086 (0.914)&2.8146 (0.012) &0.995\\\hline
       default (batch size=512)  & 23.9094 (0.469) & 2.8298 (0.004)&0.998\\
       batch size=64  &39.5403 (1.238)&2.9670 (0.016)&0.996 \\\hline
       default (clipping=0.0)  & 23.9094 (0.469) & 2.8298 (0.004)&0.998\\
       clipping=1.0  &24.3836 (0.481)&2.8207 (0.006)&0.998 \\\hline
       default (momentum=0.95)  & 23.9094 (0.469) & 2.8298 (0.004)&0.998\\
       momentum=0.9  & 23.8508      (0.497)&2.8432 (0.007)&0.998\\
       \hline
       default (weight decay=0.01)  & 23.9094 (0.469) & 2.8298 (0.004)&0.998\\
       weight decay=0.0  &25.1154 (0.416) &2.8158 (0.005)&0.999\\\hline
    \end{tabular}
    \caption{Ablation study of \Cref{abstract:4} over weight decay, gradient clipping, momentum coefficient, batch size, and random seeds. High $R^2$ score supports our $O(1/\sqrt{T})$ loss convergence with $1/\sqrt{T}$-scaled learning rate.}
    \label{tab:ablation}
\end{table}

Here, the default setting is $\etar=1.0$ for GPT2 0.1B, batch size 512(*1024 context length), weight decay 0.01, momentum 0.95, random seed 1337, clipping norm 0.0, Muon-NSGD optimizer and cosine learning rate schedule (with 2\% warmup), the same as in \Cref{sec:experiments}.

\end{document}